\documentclass{article}

 \usepackage[preprint]{neurips_template}

\usepackage[utf8]{inputenc} % allow utf-8 input
\usepackage[T1]{fontenc}    % use 8-bit T1 fonts
\usepackage{hyperref}       % hyperlinks
\usepackage{url}            % simple URL typesetting
\usepackage{booktabs}       % professional-quality tables
\usepackage{amsfonts}       % blackboard math symbols
\usepackage{nicefrac}       % compact symbols for 1/2, etc.
\usepackage{microtype}      % microtypography
\usepackage{xcolor}         % colors

\usepackage{amsmath}
\usepackage{amssymb}
\usepackage{mathtools}
\usepackage{amsthm}
\usepackage{cases}

\usepackage{algorithm}
\usepackage{algorithmic}

\usepackage{thmtools}
\usepackage[capitalize,noabbrev]{cleveref}
\usepackage{enumitem}
\usepackage[most]{tcolorbox}
\usepackage{multirow}
\theoremstyle{plain}
\newtheorem{theorem}{Theorem}[section]
\newtheorem{proposition}[theorem]{Proposition}
\newtheorem{lemma}[theorem]{Lemma}

\theoremstyle{definition}

\newtheorem{assumption}[theorem]{Assumption}
\theoremstyle{remark}
\newtheorem{remark}[theorem]{Remark}

\crefname{definition}{Definition}{Definitions}
\crefname{assumption}{Assumption}{Assumptions}
\crefname{proposition}{Proposition}{Propositions}
\crefname{lemma}{Lemma}{Lemmas}
\crefname{corollary}{Corollary}{Corollaries}
\crefname{remark}{Remark}{Remarks}

\newcommand{\pmdmean}{\textsc{PMD-mean}}
\newcommand{\pmdpart}{\textsc{PMD-part}}
\newcommand{\KL}{\mathrm{KL}}
\newcommand{\TV}{\mathrm{TV}}
\newcommand{\E}{\mathbb{E}}
\newcommand{\Var}{\mathrm{Var}}
\newcommand{\cW}{W} 
\DeclareMathOperator*{\argmax}{argmax}
 
\newcommand{\abs}[1]{\left|#1\right|}

\title{Approximation of Log-Partition Function in Policy Mirror Descent Induces Implicit Regularization for LLM Post-Training}

\author{%
  Zhenghao Xu\textsuperscript{1}\thanks{Work done during internship at Amazon. Emails: \texttt{\{zhenghaoxu,tourzhao\}@gatech.edu}}
  ~~~
  Qin Lu\textsuperscript{2}
  ~~~
  Changlong Yu\textsuperscript{2}
  ~~~
  Tuo Zhao\textsuperscript{1} 
  \\
\textsuperscript{1}Georgia Institute of Technology\quad 
\textsuperscript{2}Amazon
}

\begin{document}

\maketitle

\begin{abstract}
Policy mirror descent (PMD) provides a principled framework for reinforcement learning (RL) by iteratively solving KL-regularized policy improvement subproblems. While this approach has been adopted in training advanced LLMs such as Kimi K1.5/K2, the ideal closed-form PMD updates require reliable partition function estimation, a significant challenge when working with limited rollouts in the vast action spaces of LLMs.

We investigate a practical algorithm, termed \pmdmean, that approximates the log-partition term with the mean reward under the sampling policy and performs regression in log-policy space. Specifically, we characterize the population solution of \pmdmean\ and demonstrate that it implicitly optimizes mirror descent subproblems with an adaptive mixed KL--$\chi^2$ regularizer. This additional $\chi^2$ regularization constrains large probability changes, producing more conservative updates when expected rewards are low and enhancing robustness against finite-sample estimation errors.

Experiments on math reasoning tasks show that \pmdmean\ achieves superior performance with improved stability and time efficiency. These findings deepen our understanding of \pmdmean\ and illuminate pathways toward principled improvements in RL algorithms for LLMs. 
Code is available at \url{https://github.com/horizon-rl/OpenKimi}.
\end{abstract}

\section{Introduction}
\label{sec:intro}

Reinforcement learning (RL) has become a standard paradigm for enhancing post-training of large language models (LLMs) on reasoning tasks and agentic objectives \citep{jaech2024openai,guo2025deepseek}. Despite diverse implementation approaches, most RL algorithms can be formalized as \emph{regularized policy improvement}, an iterative method that updates policies to maximize rewards while maintaining proximity to reference policies.

Policy mirror descent (PMD, \citealt{geist2019theory,tomarmirror}) provides a canonical formalization of this approach by iteratively solving KL-regularized improvement subproblems. In theory, these subproblems admit elegant closed-form solutions that reweight the current policy and renormalize using the partition function. In practice, however, reliably estimating this partition function and fitting the ideal target from finite rollouts presents significant challenges, particularly in the large action space of LLM post-training.

A common approach to solving KL-regularized subproblems involves applying policy gradient methods \citep{williams1992simple} directly to the regularized objective, either by incorporating regularization into the reward function or adding an explicit KL penalty. Methods such as TRPO \citep{schulman2015trust}, PPO \citep{schulman2017proximal}, RLOO \citep{ahmadian2024back}, and GRPO \citep{shao2024deepseekmath} use rollout samples to construct surrogate losses and perform optimally when rollouts are from the current policy, i.e., on-policy.

However, modern efficient RL implementations increasingly leverage large generation batches or asynchronous rollouts to prevent computational bottlenecks from long-tail generations \citep{noukhovitch2024asynchronous,fu2025areal}. 
These approaches typically incur a \emph{staleness tax}: the sampling policy predates the policy being updated, creating a fundamental training/inference mismatch. This mismatch introduces instability that practitioners attempt to mitigate through importance weighting with clipping or similar heuristics \citep{yao2025offpolicy,liu-li-2025-rl-collapse}. While partially effective, these remedial techniques substantially complicate both implementation and theoretical analysis.

This paper investigates an alternative minimalist approach popularized by Kimi K1.5/K2 \citep{team2025kimik15,team2025kimik20} that fundamentally reframes the problem. 
Rather than attempting to mitigate off-policy-ness through complex correction mechanisms and heuristics, this algorithm adopts an off-policy regression perspective on PMD.
Specifically, instead of fitting the exact partition-normalized target, the method approximates the log-partition term with the \emph{mean reward} under the sampling policy and fits a regression target directly in log-policy space. We refer to this practical algorithm as \pmdmean\ (or ``Kimi-style PMD'').

While this mean-reward approximation remains accurate under strong regularization conditions, it can diverge significantly from the partition-normalized update when using smaller regularization, typical in practice. This divergence raises a fundamental question:

\begin{center}
\emph{What does \pmdmean\ optimize exactly, and what are the algorithmic consequences of that objective?}
\end{center}

\textbf{Our results.} We address this fundamental question by deriving a closed-form characterization of the \pmdmean\ population solution. While the ideal KL-regularized PMD update produces a standard Boltzmann reweighting, our analysis reveals that \pmdmean\ generates an update involving the Lambert-$W$ function.

Furthermore, we show that this update is mathematically equivalent to performing mirror descent with a \emph{mixed} KL--$\chi^2$ regularizer, where the $\chi^2$ weight \emph{adapts} dynamically based on the mean reward under the current policy. This additional $\chi^2$ term imposes stronger penalties on probability changes compared to KL alone, and the effect is particularly pronounced when the mean reward is low. This adaptive regularization effectively moderates the convergence rate during the early phases of training, providing a principled explanation for the algorithm's empirical stability.

Our further analysis demonstrates that, compared to fitting the partition-normalized target (\pmdpart), \pmdmean\ exhibits significantly reduced sensitivity to finite-sample errors when rollouts are limited. This characteristic substantially decreases the risk of overfitting to misestimated targets. 
The finding provides a theoretical explanation for \pmdmean's enhanced stability in practical applications: The implicitly induced $\chi^2$ regularization introduces additional robustness that is valuable in the data-constrained scenarios typical of LLM post-training.

\textbf{Contributions.}
We make the following contributions:

$\bullet$ \textbf{Exact characterization of \pmdmean.} We derive the closed-form \pmdmean\ solution in policy space and establish its equivalence to a mirror-descent subproblem with an adaptive mixed KL--$\chi^2$ regularizer.

$\bullet$ \textbf{Regularization and stability mechanism.} We demonstrate that the induced $\chi^2$ term provides additional control over probability ratios, offering substantial regularization even when the nominal KL coefficient is minimal.

$\bullet$ \textbf{Convergence analysis.} Under standard assumptions, we develop an inexact-PMD style convergence analysis that distinguishes \pmdmean\ from \pmdpart\ and characterizes their separations.

$\bullet$ \textbf{Experimental validation.} Through experiments on math reasoning tasks, we empirically confirm \pmdmean's enhanced efficiency and stability while demonstrating its performance advantages over the standard GRPO.

% \textbf{Organization.}
% \Cref{sec:prelim} presents the problem setup and reviews the PMD framework.
% \Cref{sec:pmdmean} derives the closed-form solution of \pmdmean\ and develops the mixed-divergence reformulation.
% \Cref{sec:theory} analyzes the convergence properties and theoretical implications.
% \Cref{sec:experiments} presents supporting experimental results on LLMs.
% \Cref{sec:related} discusses connections to related work, and \Cref{sec:conclusion} provides concluding remarks.
% Complete proofs and detailed experimental details are provided in the Appendix.

% ======================================================================
\section{Preliminaries}
\label{sec:prelim}

For simplicity, we formulate RL for LLM post-training as a contextual bandit problem. Let $x\in\mathcal{X}$ be an input prompt (state) and $y\in\mathcal{Y}$ represent a generated response (action), with $r(x,y)\in[0,1]$ denoting a bounded reward function.
A language model policy $\pi\colon\mathcal{X}\to\Delta(\mathcal{Y})$, which maps states to distributions over actions, induces the expected reward:
\begin{align*}
    J(\pi) \coloneqq \mathbb{E}_{x\sim\mathcal{D}} \mathbb{E}_{y\sim \pi(\cdot\mid x)} \big[r(x,y)\big],
\end{align*}
where $\mathcal{D}$ is the distribution over prompts (e.g., a dataset). 
In practice, LLM policies are implemented as compositions of token-wise softmax distributions, which assign non-zero probability to every token in the vocabulary. Consequently, we can reasonably assume that $\pi$ has full support over the entire action space of possible responses, that is, $\pi(y\mid x)>0$ for all $x\in\mathcal{X}$ and $y\in\mathcal{Y}$.

\textbf{Policy mirror descent.} 
At global step $t$, KL-regularized policy mirror descent (PMD; \citealt{geist2019theory,tomarmirror}) updates policy $\pi_t$ by solving the following optimization problem for each state $x$:
\begin{align}
    \pi_{t+1}(\cdot\mid x) = 
    &\argmax_{\pi(\cdot\mid x)\in \Delta(\mathcal{Y})}
    \mathbb{E}_{y\sim \pi(\cdot\mid x)}[r(x,y)]-
    \tau\cdot\text{KL}\left(\pi(\cdot\mid x)~\|~\pi_t(\cdot\mid x)\right),
    \label{eq:pmd_kl_subproblem}
\end{align}
where $\tau>0$ is the regularization parameter that controls the strength of regularization. The KL divergence between distributions $p$ and $q$ over $\mathcal{Y}$ is $\text{KL}(p~\|~q)\coloneqq\mathbb{E}_{y\sim p}\left[\log\frac{p(y)}{q(y)}\right]$.

This KL-regularized optimization problem in \eqref{eq:pmd_kl_subproblem} admits the following unique closed-form solution:
\begin{align}
    \pi_{t+1}(y\mid x)
    =\frac{\pi_t(y\mid x)\exp(r(x,y)/\tau)}{Z_t(x)},
    \label{eq:boltzmann_update}
\end{align}
where $Z_t(x)\coloneqq \mathbb{E}_{y\sim \pi_t(\cdot\mid x)}\left[e^{r(x,y)/\tau}\right]$ is the partition function that ensures proper normalization. This update resembles a Boltzmann distribution that exponentially reweights the previous policy according to the reward signal.

\textbf{Fitting the target by regression.}
The ideal update in \eqref{eq:boltzmann_update} is computationally infeasible in large action spaces as it requires partition function evaluation and per-action probability assignment. 
% evaluating the partition function over all possible responses. 
A direct off-policy approach to approximate this update is to fit the target in log-policy space using squared regression. After collecting rollouts $y\sim \pi_t(\cdot\mid x)$ and evaluating rewards $r(x,y)$, we define the state-dependent target log-ratio:
\begin{align*}
    s^\star_{\mathrm{part}}(x,y)
    \coloneqq\log \frac{\pi_{t+1}(y\mid x)}{\pi_t(y\mid x)}
    = \frac{r(x,y)}{\tau} - \log Z_t(x),
\end{align*}
where $\pi_{t+1}$ is defined according to Equation \eqref{eq:boltzmann_update}. 
This leads to the following squared regression loss: 
\begin{align}
    &\mathcal{L}_{\mathrm{part}}(\pi)
    \coloneqq\E_{x\sim \mathcal{D}}\E_{y\sim \pi_t(\cdot\mid x)}\left[\frac{1}{2}\Big(\log \frac{\pi(y\mid x)}{\pi_t(y\mid x)} - s^\star_{\mathrm{part}}(x, y)\Big)^2\right].
    \label{eq:L_partition}
\end{align}
If the policy class can represent the target distribution exactly (i.e., is realizable) and $Z_t(x)$ is known precisely, then minimizing \eqref{eq:L_partition} recovers the optimal policy update. 
However, in practice, $Z_t(x)$ must be estimated from the same finite set of rollouts, which introduces significant estimation error. 
For large action spaces and small regularization parameters $\tau$, this estimation can be highly unstable, leading to pathological update behavior (see \Cref{sec:experiments}). 
We refer to this direct fitting approach as \pmdpart.

\textbf{\pmdmean: Log-partition approximation.}
Instead of fitting $s^\star_{\mathrm{part}}$, \citet{team2025kimik15} propose an alternative approach that approximates the log-partition function with the average reward, resulting in a modified loss function. 
Specifically, we define the advantage function $\Delta(x,y)$ under $\pi_t$ with mean reward baseline:
\begin{align*}
\Delta(x,y) \coloneqq r(x,y) - \mathbb{E}_{y^\prime\sim \pi_t(\cdot\mid x)}[r(x,y^\prime)].
\end{align*}
With this definition, the regression objective becomes:
\begin{align}
    &\mathcal{L}_{\mathrm{mean}}(\pi)
    \coloneqq\E_{x\sim \mathcal{D}}\E_{y\sim \pi_t(\cdot\mid x)}\left[\frac{1}{2}\left(\log \frac{\pi(y\mid x)}{\pi_t(y\mid x)}-\frac{\Delta(x,y)}{\tau}\right)^2\right].
    \label{eq:L_mean}
\end{align}
We refer to this practical variant as \pmdmean, which has been adopted to train advanced models such as Kimi K1.5 and K2 \citep{team2025kimik15,team2025kimik20}. 
In practice, the expected reward $\mathbb{E}_{y^\prime\sim \pi_t(\cdot\mid x)}[r(x,y^\prime)]$ can be efficiently estimated using a per-prompt Monte Carlo average over sampled responses.

\section{Implicit Regularization of \pmdmean}
\label{sec:pmdmean}

The approximation adopted by \pmdmean\ is accurate when $\tau\to\infty$. 
However, a large $\tau$ will significantly slow down convergence, and the average reward deviates significantly from the log-partition function when $\tau$ becomes small (\Cref{fig:partition_vs_mean}, left). 
Therefore, the solution of \eqref{eq:L_mean} may differ significantly from the ideal target $\pi_{t+1}$ in \eqref{eq:boltzmann_update}, and thus no longer corresponds to the solution of the KL subproblem \eqref{eq:pmd_kl_subproblem}, even with infinite samples provided (\Cref{fig:partition_vs_mean}, right).

\subsection{Exact Solution of \pmdmean}
\label{sec:closed_form}
To understand the actual objective of \pmdmean, we characterize the population minimizer of \eqref{eq:L_mean}. 
For simplicity, we omit $x$ and consider a single state, writing $\pi_t(y)$ and $\Delta_y$.
\begin{theorem}[\pmdmean\ solution]
    \label{thm:pmdmean_solution}
    Assume $\pi_t(y)>0$ for all $y\in\mathcal{Y}$.
    Let $\Delta_y \coloneqq r(y) - \E_{y^\prime\sim \pi_t}[r(y^\prime)]$ denote the mean-baseline advantage.
    Then the unique minimizer of \eqref{eq:L_mean} over the probability simplex satisfies
    \begin{align}
        \pi_{t+1}(y)
        =\pi_t(y)\exp\Big(\frac{\Delta_y}{\tau}-\cW\Big(\frac{\lambda}{\tau^2}\exp\Big(\frac{\Delta_y}{\tau}\Big)\Big)\Big),
        \label{eq:pmdmean_lambertW_form}
    \end{align}
    where $\cW(\cdot)$ is the principal branch of the Lambert-$W$ function (inverse of $f(w)=w\cdot e^w$) and $\lambda\ge 0$ is a normalization constant chosen such that $\sum_y \pi_{t+1}(y)=1$.
    Moreover, defining $A\coloneqq \E_{\pi_t}[\exp(\Delta_y/\tau)]$ and $B\coloneqq \E_{\pi_t}[\exp(2\Delta_y/\tau)]$, the normalization constant satisfies
    \begin{align}
        \tau^2\frac{A(A-1)}{B}\leq\lambda\leq\tau^2\log A.
        \label{eq:lambda_bounds}
    \end{align}
    For binary rewards $r\in\{0,1\}$ with $p=\E_{\pi_t}[r(y)]$,
    \begin{align}
        \lambda\approx\begin{cases}
            \frac{1}{2}p(1-p), & \tau\to\infty\\
            \tau p(1-p), &\tau\to 0.
        \end{cases}
        \label{eq:lambda_asymptotics}
    \end{align}
\end{theorem}
By \Cref{thm:pmdmean_solution}, the solution of \pmdmean\ has its action probabilities heterogeneously normalized by the Lambert-$W$ function, as opposed to the KL solution \eqref{eq:boltzmann_update} where the normalization term is the log-partition function that is independent of the action $y$. 
Given the monotonicity of $W$, actions with larger $\Delta_y$ will get their probability suppressed compared to the KL solution, while actions with smaller $\Delta_y$ will not be punished as hard. 
Therefore, \pmdmean\ update is less aggressive than \pmdpart.

More precisely, suppose the reward is binary and the average reward under $\pi_t$ is $p$. 
We consider the ratio $\pi_{t+1}/\pi_t$ on positive and negative actions when $\tau$ is small. 
For positive actions, $r(y)=1$, \pmdmean\ yields 
\begin{align}
    \frac{\pi^{\mathrm{mean}}_{t+1}(y)}{\pi_t(y)}
    =\frac{1}{p}-\frac{1-p}{p}e^{-p/\tau}(1+o(1)),
    \label{eq:ratio_pos_mean}
\end{align}
while \pmdpart\ yields 
\begin{align}
    \frac{\pi^{\mathrm{part}}_{t+1}(y)}{\pi_t(y)} 
    &= \frac{1}{p} - \frac{1-p}{p^2}e^{-1/\tau} + O(e^{-2/\tau}).
    \label{eq:ratio_pos_part}
\end{align}
While both ratios approach $1/p$ from below when $\tau\to 0$, the gap is much larger in \pmdmean\ when $p$ is small (e.g., early phase of training), hence giving a more conservative update. 
For negative actions, the separation is clearer: For $r(y)=0$, \pmdmean\ yields 
\begin{align}
    \frac{\pi^{\mathrm{mean}}_{t+1}(y)}{\pi_t(y)}
    =e^{-p/\tau}(1+o(1)),
    \label{eq:ratio_neg_mean}
\end{align}
while \pmdpart\ yields 
\begin{align}
    \frac{\pi^{\mathrm{part}}_{t+1}(y)}{\pi_t(y)} 
    =\frac{1}{p}e^{-1/\tau}+O(e^{-2/\tau}).
    \label{eq:ratio_neg_part}
\end{align}
When $p$ is small, the difference is significant, as illustrated in \Cref{fig:ratio}. 
Full derivations are provided in \Cref{sec:appendix-ratio}.

\begin{figure}[t]
    % \vspace{-0.1in}
    \begin{center}
        \begin{minipage}[c]{0.4\columnwidth}
            \centering
            \includegraphics[width=\linewidth]{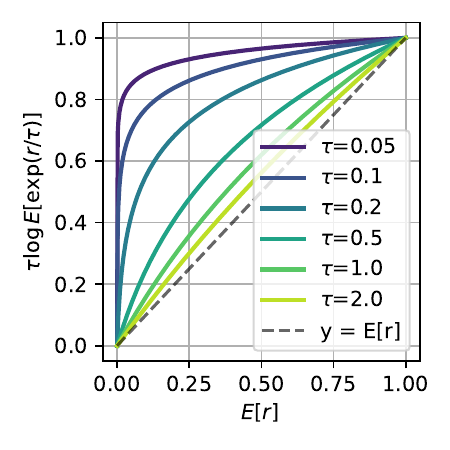}
        \end{minipage}
        \begin{minipage}[c]{0.35\columnwidth}
            \centering
            \includegraphics[width=\linewidth]{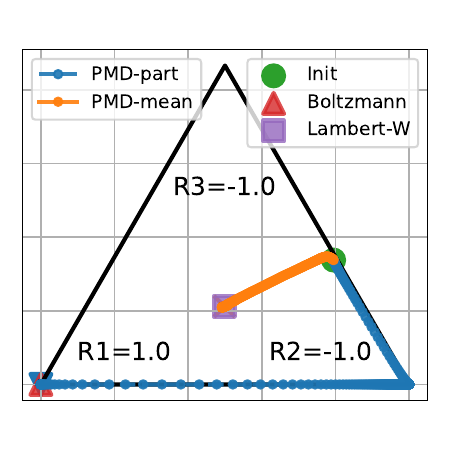}
        \end{minipage}
        \vspace{-0.1in}
        \caption{
        Left: Scaled log-partition function vs average reward assuming binary rewards. The gap is significant for moderate $\tau$.
        Right: Illustration of \pmdmean\ and \pmdpart\ converging to different subproblem solutions in the probability simplex.
        }
        \label{fig:partition_vs_mean}
    \end{center}
    \vspace{-0.1in}
\end{figure}

% ======================================================================
\subsection{\pmdmean\ as Mirror Descent with Mixed KL--$\chi^2$ Regularization}
\label{sec:mixed_div}

The Lambert-$W$ closed form is useful for analysis, while a more transparent insight is that \pmdmean\ is \emph{exactly} solving a different regularized policy improvement problem.
Let $\chi^2(p~\|~q)\coloneqq\E_{y\sim q}\big[(\frac{p(y)}{q(y)}-1)^2\big]$ denote the $\chi^2$-divergence between two distributions $p$ and $q$ over $\mathcal{Y}$.
The following proposition shows that the \pmdmean\ update is equivalent to mirror descent with an additional $\chi^2$ penalty.

\begin{proposition}[Equivalent mixed KL--$\chi^2$ subproblem]
\label{prop:mixed_subproblem}
    Fix a state $x$ (omitted for brevity).
    Let $\pi_{t+1}$ be the \pmdmean\ population solution in \Cref{thm:pmdmean_solution} and $\lambda$ be the same normalization constant.
    Then $\pi_{t+1}$ is the solution to
    \begin{align}
        \pi_{t+1}
        =&\argmax_{\pi\in\Delta(\mathcal{Y})}
        \E_{y\sim \pi}\big[r(y)\big]
        -\tau\KL(\pi~\|~\pi_t)-\frac{\lambda}{2\tau}\chi^2(\pi~\|~\pi_t).
        \label{eq:mixed_subproblem}
    \end{align}
\end{proposition}
\Cref{prop:mixed_subproblem} follows directly from the KKT conditions of both problems (see \Cref{sec:proof_mixed}). 
The $\chi^2$ penalty directly suppresses large policy ratio spikes and provides a stronger regularization compared to KL, as $\KL(p~\|~q)\leq \chi^2(p~\|~q)$ for $p,q$ with full support. 
Moreover, \Cref{thm:pmdmean_solution} shows the effective strength $\lambda/\tau$ is \emph{adaptive}.
For binary rewards, $\lambda/\tau$ remains $O(1)$ even as $\tau\to 0$, which implies \pmdmean\ still regularizes updates when the nominal KL regularization is small.

\begin{remark}[Connection to $\chi^2$ preference optimization]    
The mixed KL--$\chi^2$ regularizer has appeared in $\chi^2$-regularized preference optimization ($\chi$PO, \citealt{huang2024correcting}), which provides provable guarantees to avoid overoptimization in KL-regularized preference optimization. 
Our results show that \pmdmean\ can be interpreted as an online version of this idea, with an adaptive coefficient $\lambda$ tied to the rollout reward distribution. 
\end{remark}
\begin{remark}[Policy ratios compared to \citealt{huang2024correcting}]    
\citet{huang2024correcting} show in their Proposition 4.2 that mixed KL--$\chi^2$ regularized problem yields (in our notations)
\begin{align}\label{eq:huang_mixed_ratio}
    \exp\left(-\frac{R_{\max}}{\tau}\right)\lesssim\frac{\pi^\mathrm{mix}_{t+1}}{\pi_t}\lesssim 1+\frac{R_{\max}}{\tau},
\end{align}
while the KL-regularized problem yields 
\begin{align}\label{eq:huang_kl_ratio}
    \exp\left(-\frac{R_{\max}}{\tau}\right)\lesssim\frac{\pi^\mathrm{KL}_{t+1}}{\pi_t}\lesssim \exp\left(\frac{R_{\max}}{\tau}\right),
\end{align}
where $r\in[0, R_{\max}]$ is the bound of rewards. 
This highlights a worst-case {polynomial vs. exponential} contrast in $1/\tau$ for the upper policy ratio control. 
However, these bounds are uniform in $R_{\max}$ and can be vacuous in regimes relevant to our binary reward analysis. 
Particularly, \eqref{eq:huang_kl_ratio} neglects the partition normalizer that also grows with $1/\tau$. 
In contrast, our setting admits substantially sharper and distribution-dependent behavior for small $\tau$. 
As shown in \eqref{eq:ratio_pos_mean} and \eqref{eq:ratio_pos_part}, both ratios share the same constant upper bound $1/p$, with different exponential rates of approaching this upper bound. 
Therefore, the gap between mixed and KL divergence in our setting is more of a distinction in the distribution-dependent exponential rates, i.e., $O(e^{p/\tau})$ vs. $O(e^{1/\tau})$, instead of polynomial vs. exponential. 
\end{remark}

\begin{figure}[t]
     % \vspace{-0.1in}
  \begin{center}
    \centerline{\includegraphics[width=0.75\columnwidth]{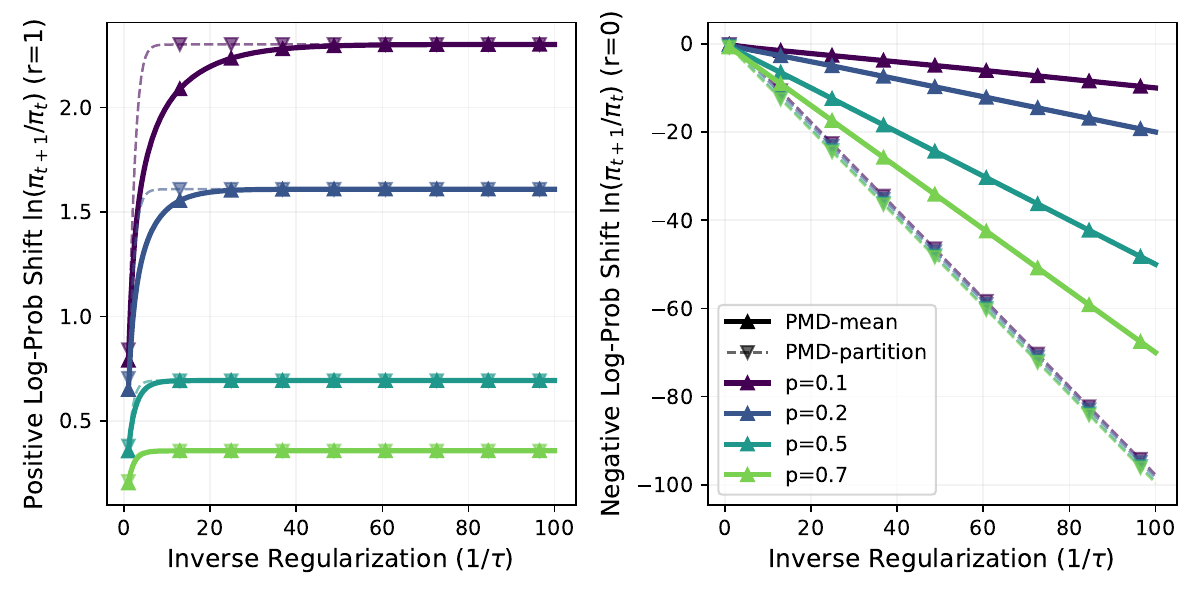}}
    \vspace{-0.1in}
    \caption{
      The (log) probability ratio of updates in \pmdmean\ is more conservative than that in \pmdpart\ for binary rewards. 
    }
    \label{fig:ratio}
  \end{center}
  \vspace{-0.2in}
\end{figure}

\section{Implications on Convergence}
\label{sec:theory}
We present an inexact-PMD convergence analysis from the regression view to illustrate the implications of the implicit regularization on convergence. 
For clarity, we still suppress $x$ and analyze one state. The extension to averaging over the prompt distribution $x\sim\mathcal{D}$ is standard.

\subsection{One-Step Policy Improvement}
\label{sec:convergence}
Recall $J(\pi)\coloneqq \E_{y\sim \pi}[r(y)]$ where $r(y)\in[0,1]$.
At iteration $t$, let $\pi^\star_{t+1}$ denote the ideal target update in \pmdpart\ or \pmdmean,
and define its target log-ratio $s^\star(y)\coloneqq \log\frac{\pi^\star_{t+1}(y)}{\pi_t(y)}$. 
Let $\Pi$ be the policy class and define, for each $\pi\in\Pi$, $s_\pi(y)\coloneqq \log\frac{\pi(y)}{\pi_t(y)}$.
The goal of \pmdpart\ and \pmdmean\ is to fit the ideal target $s^\star$ with $s_\pi$ by minimizing the population loss
\begin{align}\label{eq:population-loss}
    \mathcal{L}_t(\pi)
    &\coloneqq
    \frac{1}{2}\E_{y\sim\pi_t}\big[(s_\pi(y)-s^\star(y))^2\big].
\end{align}

In practice, $\pi^\star_{t+1}$ depends on population quantities (e.g., $\E_{\pi_t}[r]$ or $\log Z_t$) that are not exactly available, so one instead forms an estimated target update $\widetilde{\pi}^\star_{t+1}$ via finite MC samples.
Given i.i.d. samples $y_1,\dots,y_n\sim \pi_t$, we minimize the empirical loss
\begin{align}\label{eq:empirical-loss}
    \widehat{\mathcal{L}}_t(\pi)
    &\coloneqq
    \frac{1}{2n}\sum_{i=1}^n (s_\pi(y_i)-\widetilde{s}_{-i}^\star(y_i))^2,
\end{align}
where the target at $y_i$ becomes the leave-one-out (LOO) estimated version $\widetilde{s}_{-i}^\star(y_i)$: 
\begin{numcases}{\widetilde{s}_{-i}^\star(y_i)=}
    \frac{1}{\tau}\Bigl(r(y_i)-\frac{1}{n-1}\sum_{j\neq i}r(y_j)\Bigr) & \text{(\pmdmean)}\label{eq:loo-s-mean}\\
    \frac{r(y_i)}{\tau}-\log\Bigl(\frac{1}{n-1}\sum_{j\neq i}e^{r(y_j)/\tau}\Bigr) & \text{(\pmdpart)}\label{eq:loo-s-part}
\end{numcases}

Let $\widehat{\pi}_{t+1}$ be an approximate ERM solution and set $\pi_{t+1}\coloneqq \widehat{\pi}_{t+1}$. 
We make the following assumptions for analysis. 

\begin{assumption}[Realizability]
\label{asm:realizable_plugin}
    $\pi^\star_{t+1}\in\Pi$.
\end{assumption}

\begin{assumption}[Bounded optimization error]
\label{asm:opt_err}
    $\widehat{\mathcal{L}}_{t}(\widehat{\pi}_{t+1})\leq \inf_{\pi\in\Pi}\widehat{\mathcal{L}}_t(\pi)+\epsilon_{\mathrm{opt}}$.
\end{assumption}

\begin{assumption}[Bounded log-ratio]
\label{asm:bounded_logratio}
    There exist $B,B_+>0$ such that for all $\pi\in\Pi$ and $y\in\mathcal{Y}$,
    \begin{align}
        s_\pi(y)\leq B_+,
        \quad
        \abs{s_\pi(y)}\leq B.
        \label{eq:bounded_logratio_conv}
    \end{align}
\end{assumption}

\begin{assumption}[Finite policy class]
\label{asm:finite_class}
    $\abs{\Pi}<\infty$.
\end{assumption}

\Cref{asm:realizable_plugin} is standard in the sample-efficient RL literature \citep{foster2023foundations}. 
\Cref{asm:opt_err} is a generic assumption that allows focusing on statistical rates without involving overcomplicated subproblem optimization dynamics.
\Cref{asm:bounded_logratio} can be achieved by restricting the policy class via clipping. 
\Cref{asm:finite_class} is for simplicity, and one can extend it to general policy classes with $\abs{\Pi}$ replaced by other complexity metrics, e.g., covering number.  
Under these assumptions, we have the following lemma that characterizes the error of the approximate ERM solution $\widehat{\pi}_{t+1}$. 
\begin{lemma}[Empirical minimization]\label{thm:erm_loo}
    For global step $t$, suppose $y_1,\dots,y_n\sim\pi_t$ are i.i.d. samples. 
    Let $\mathcal{L}_t(\pi)$ and $\widehat{\mathcal{L}}_t(\pi)$ denote the population and empirical losses defined in \eqref{eq:population-loss} and \eqref{eq:empirical-loss}, respectively. 
    Define the target mismatch
    \begin{align*}
        \Delta_i \coloneqq \widetilde{s}^\star_{-i}(y_i)-s^\star(y_i),
        \quad
        \overline{\Delta^2}\coloneqq \frac{1}{n}\sum_{i=1}^n \Delta_i^2.
    \end{align*}
    Under \Cref{asm:realizable_plugin,asm:opt_err,asm:bounded_logratio,asm:finite_class}, for any $\delta\in(0,1)$, with probability at least $1-\delta$,
    \begin{align}
    \label{eq:erm_master_bound}
        \mathcal{L}_t(\widehat{\pi}_{t+1})
        \lesssim
        \frac{B^2\log(\abs{\Pi}/\delta)}{n} + \epsilon_{\mathrm{opt}} + \overline{\Delta^2}.
    \end{align}
\end{lemma}
The proof is provided in \Cref{sec:proof_erm_loo}. 
The ERM solution error is then translated into the gap between $J(\pi^\star_{t+1})$ and $J(\widehat{\pi}_{t+1})$ that affects the one-step policy improvement. 
\begin{theorem}[One-step policy improvement]
    \label{cor:convergence_one_step}
    Under \Cref{asm:realizable_plugin,asm:opt_err,asm:bounded_logratio,asm:finite_class}, suppose the ideal population update $\pi^\star_{t+1}$ satisfies, for some $\eta_t\in(0,1]$,
    \begin{align}
        1 - J(\pi^\star_{t+1})
        \leq
        (1-\eta_t)\big(1-J(\pi_t)\big).
        \label{eq:ideal_improvement}
    \end{align}
    Let $\pi_{t+1}\coloneqq\widehat{\pi}_{t+1}$ and $\overline{\Delta^2}$ be as in \Cref{thm:erm_loo}. 
    Then for $\delta\in(0,1)$, with probability at least $1-\delta$,
    \begin{align}
        &1 - J(\pi_{t+1})
        \leq (1-\eta_t)\big(1-J(\pi_t)\big)+O\Big(e^{B_+/2}\Big(B\sqrt{\frac{\log(\abs{\Pi}/\delta)}{n}}+\sqrt{\epsilon_{\mathrm{opt}}}
        +\sqrt{\overline{\Delta^2}}\Big)\Big).
        \label{eq:one_step_improvement}
    \end{align}
\end{theorem}
% The proof is provided in \Cref{sec:proof_cor_one_step}. 

\subsection{Instantiation and Separation}
\label{sec:instantiate_eta_B}
We now specialize \Cref{cor:convergence_one_step} to the binary reward model $r(y)\in\{0,1\}$ with $p_t \coloneqq J(\pi_t) \in (0,1)$ in the small $\tau>0$ regime, and instantiate
(i) the {ideal improvement rate} $\eta_t$ in \eqref{eq:ideal_improvement}, 
(ii) the {bounded log-ratio} constants $(B,B_+)$ in \Cref{asm:bounded_logratio}, and 
(iii) the {target estimation error} $\overline{\Delta^2}$. 
% We focus on the small $\tau$ regime where the separation is clearer. 
These quantities highlight a separation between \pmdpart\ and \pmdmean\ at the early phase of training where $p_t$ is small:
\pmdpart\ has a faster \emph{ideal} convergence rate, while \pmdmean\ can be more \emph{robust} to statistical errors under finite rollouts.

\subsubsection{Ideal Convergence Rate $\eta_t$}
\label{sec:eta_instantiation}
The ideal improvement rate is a direct consequence of our analysis in \Cref{sec:pmdmean}, and reveals the behavior when the rollout sample size $n$ is large. 
\begin{proposition}[Ideal contraction for \pmdmean\ with small $\tau$]
\label{prop:eta_mean_asymp}
    Let $\pi^\star_{t+1}$ be the ideal update \eqref{eq:pmdmean_lambertW_form} with binary rewards. Then as $\tau\to 0$ we have
    \begin{align}
        \eta_t^{\mathrm{mean}}
        =1-\exp\left(-\frac{p_t}{\tau}\right)\big(1+o(1)\big).
        \label{eq:eta_mean_asymp}
    \end{align}
\end{proposition}

\begin{proposition}[Ideal contraction for \pmdpart]
\label{prop:eta_part_exact}
Let $\pi^\star_{t+1}$ be the ideal update \eqref{eq:boltzmann_update} with binary rewards.
Then \eqref{eq:ideal_improvement} holds with
\begin{align}
    \eta_t^{\mathrm{part}}
    % =\frac{p_t(e^{1/\tau}-1)}{1+p_t(e^{1/\tau}-1)}
    =1-\frac{1}{1-p_t+p_te^{1/\tau}}.
    \label{eq:eta_part_exact}
\end{align}
\end{proposition}
The proofs are provided in \Cref{sec:proof_eta_mean,sec:proof_eta_part}. 
The ideal convergence rates in \pmdpart\ and \pmdmean\ both approach $1$ when $\tau\to 0$, leading to one-step convergence. 
Meanwhile, \pmdpart\ approaches this rate faster than \pmdmean\ when $p_t<1$ is small.

% ----------------------------------------------------------------------
\subsubsection{Log-ratio Bounds $(B,B_+)$ Compatible with Realizability}
\label{sec:BB_instantiation}
In \Cref{cor:convergence_one_step}, the error of the inexact update depends on $B$ and $B_+$. 
We compute the smallest $(B,B_+)$ compatible with the realizability of the ideal target in the binary model.
\begin{proposition}[Log-ratios for \pmdmean\ with small $\tau$]
\label{prop:BB_mean_asymp}
    Consider binary rewards and the \pmdmean\ target $s^\star(y)=\log\frac{\pi^\mathrm{mean}_{t+1}(y)}{\pi_t(y)}$. Then for fixed $p_t\in(0,1)$ and $\tau\to 0$, the log-ratio bounds are 
    \begin{align*}
        &\exp(B_{+,t}^{\mathrm{mean}})
        =\frac{1}{p_t}-\frac{1-p_t}{p_t}e^{-p_t/\tau}(1+o(1)), \quad B_{t}^{\mathrm{mean}}
        =\frac{p_t}{\tau}+o(1).
    \end{align*}
\end{proposition}

\begin{proposition}[Log-ratios for \pmdpart\ with small $\tau$]\label{prop:BB_part_exact}
    Consider binary rewards and the \pmdpart\ target $s^\star(y)=\log\frac{\pi^\mathrm{part}_{t+1}(y)}{\pi_t(y)}$.
    Then for fixed $p_t\in(0,1)$ and $\tau\to 0$, the log-ratio bounds are 
    \begin{align*}
        &\exp(B_{+,t}^{\mathrm{part}})=\frac{1}{p_t} - \frac{1-p_t}{p_t^2}e^{-1/\tau} + O(e^{-2/\tau}),
        \quad B_{t}^{\mathrm{part}} = \frac{1}{\tau}-\log\frac{1}{p_t} + o(1).
    \end{align*}
\end{proposition}
The proofs follow from Proposition~\ref{prop:binary_ratios} in \Cref{sec:appendix-ratio}. 
By \Cref{prop:BB_part_exact,prop:BB_mean_asymp}, when $\tau$ is small, the factors in \pmdpart\ are worse than \pmdmean, and the gap is significant when $p_t$ is small. 

\subsubsection{Target Estimation Error}
It remains to compare the target estimation error. 
\begin{proposition}[Target estimation error for binary rewards]
    \label{prop:tgt_err_clean}
    Assume $r(y)\in\{0,1\}$ and define $p_t\coloneqq \E_{y\sim\pi_t}[r(y)]$.
    Let $p_{-i}\coloneqq \frac{1}{n-1}\sum_{j\neq i} r(y_j)$ and $a\coloneqq e^{1/\tau}-1$, then $Z_t=1+a p_t$.
    Fix $\delta\in(0,1)$ and define
    \begin{align*}
        \varepsilon_n(p_t,\delta)
        \coloneqq
        \sqrt{\frac{2p_t(1-p_t)\log(4n/\delta)}{n-1}}
        +
        \frac{2\log(4n/\delta)}{3(n-1)}.
    \end{align*}
    Then with probability at least $1-\delta$, the LOO mean satisfies $\max_i \abs{p_{-i}-p_t}\leq \varepsilon_n(p_t,\delta)$, and consequently: 
    
    (a) (\pmdmean) For $\widetilde{s}^\star_{-i}(y_i)$ in \eqref{eq:loo-s-mean} and $s^\star(y)=\log\frac{\pi^\mathrm{mean}_{t+1}(y)}{\pi_t(y)}$,
        \begin{align}
        \label{eq:tgt_err_mean_clean}
            \overline{\Delta^2}
            \lesssim
            \frac{\varepsilon_n(p_t,\delta)^2}{\tau^2} + \frac{p_t(1-p_t)^2}{\tau^2}.
            % =O\left(\frac{p_t(1-p_t)\log(4n/\delta)}{\tau^2n}\right).
        \end{align}
        
      (b) (\pmdpart) For $\widetilde{s}^\star_{-i}(y_i)$ in \eqref{eq:loo-s-part} and $s^\star(y)=\log\frac{\pi^\mathrm{part}_{t+1}(y)}{\pi_t(y)}$,
        \begin{align}\label{eq:tgt_err_part_clean}
            \overline{\Delta^2}
            \leq\biggl(\min\biggl\{\frac{a\cdot\varepsilon_n(p_t,\delta)}{1+a(p_t-\varepsilon_n(p_t,\delta))_+}, \frac{1}{\tau}\biggr\}\biggr)^2.
        \end{align}
\end{proposition}
The proof is provided in \Cref{sec:proof_tgt_err_clean}. 
By \Cref{prop:tgt_err_clean}, the target estimation error of \pmdmean\ consists of two parts. 
One part scales as $O\bigl(\frac{p_t(1-p_t)}{\tau^2 n}\bigr)$, while the other part $O\bigl(\frac{p_t(1-p_t)^2}{\tau^2}\bigr)$ is irreducible due to the systematic mismatch between the estimated target $\widetilde{s}_{-i}^\star$ the ideal target $s^\star$ that contains the Lambert-$W$ term. 
In particular, for positive actions, the estimated target is systematically larger than the ideal target, resulting in more aggressive improvement on these actions. 
Meanwhile, for negative actions, the estimated target is closer to the ideal, thus inherits the more conservative shrinkage at rate $\exp(-p_t/\tau)$. 

For \pmdpart, there is a critical regime where $n=\Theta(\frac{1}{p_t})$. 
When $n$ is smaller than this threshold, the bound almost scales as $O\bigl(\min\{\frac{1}{\tau^2}, \frac{{e^{2/\tau}}p_t}{n}\}\bigr)$, which is larger than the error of \pmdmean. 
When $n$ exceeds the threshold, the bound scales as $O\bigl(\frac{1-p_t}{p_t n}\bigr)$. 
This suggests that for \pmdpart, the rollout sample size should be large enough, especially at the early phase of training where $p_t$ is small. 

\Cref{fig:target_estimation,fig:target_actions} show simulations of target estimation errors, where \pmdpart\ suffers from larger error when $p_t$ and the rollout sample size $n$ are small. 
This explains that while \pmdpart\ has a faster ideal convergence rate (almost in one step for small $\tau$), it can be very unstable in practice when the rollouts are limited. 
Meanwhile, \pmdmean\ suffers less from estimation error in this regime. 

\begin{figure}[htb!]
    % \vspace{-0.1in}
    \begin{center}
    \centerline{\includegraphics[width=0.75\columnwidth]{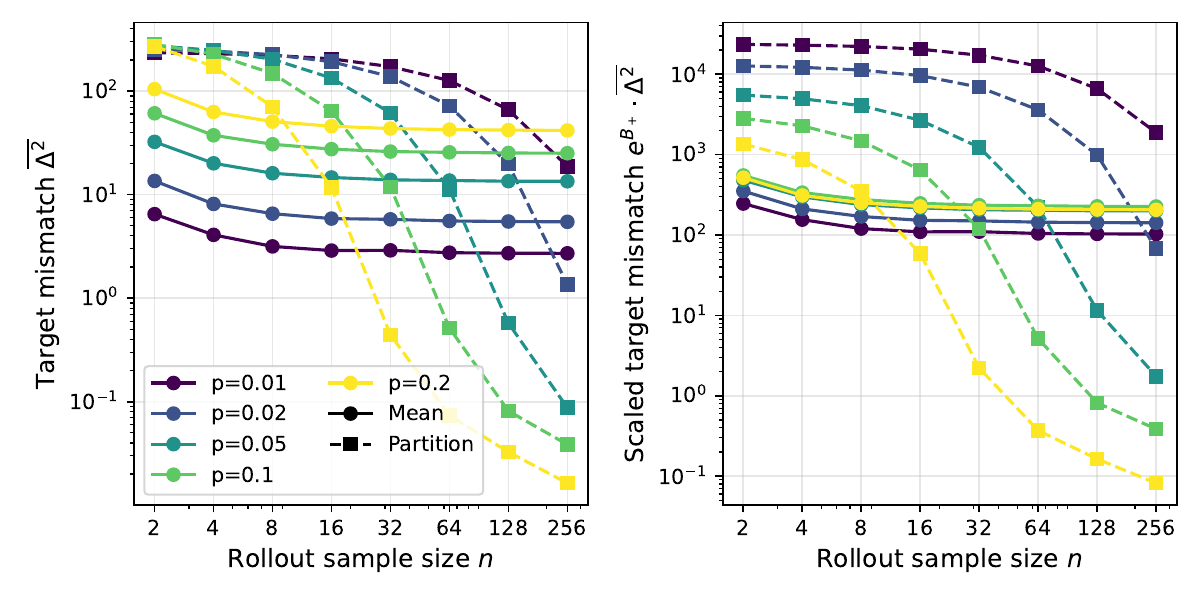}}
    \vspace{-0.1in}
    \caption{
      Target estimation error of \pmdmean\ and \pmdpart\ under $\tau=0.05$ and $p_t$ ranges from $0.01$ to $0.2$. 
      Left: the target estimation error $\overline{\Delta^2}$. Right: The scaled estimation error with corresponding prefactor $e^{B_+}$ in \eqref{eq:one_step_improvement}. The plot shows the average from $100$ random seeds. 
      When the rollout sample size $n$ is small, the error of \pmdpart\ is much larger for small $p_t$. 
      % As the $n$ increases, the estimation error of \pmdpart\ vanishes while \pmdmean\ reaches a floor. 
    }
    \label{fig:target_estimation}
    \end{center}
    \vspace{-0.15in}
\end{figure}

\subsubsection{Refined Analysis for \pmdmean}
% \begin{remark}
While the gap between the ideal target and the estimated target of \pmdmean\ does not vanish as $n\to\infty$, the minimizer of the empirical loss \eqref{eq:empirical-loss} recovers the ideal target policy in this limit, as the constraint $\mathbb{E}_{\pi_t}[e^{s_\pi}]=1$ pulls back the log-ratios from exactly fitting the advantages. 
In that large $n$ regime, \Cref{prop:tgt_err_clean} is overly pessimistic. 
We provide a refined analysis that eliminates the error floor. 
\begin{lemma}[Refined analysis for \pmdmean]\label{lem:pmdmean_refined}
    Suppose \Cref{asm:realizable_plugin,asm:opt_err,asm:bounded_logratio,asm:finite_class} hold, $r(y)\in\{0,1\}$ and define $p_t\coloneqq \E_{y\sim\pi_t}[r(y)]$.
    Let $\varepsilon_n(p_t,\delta)$ be as in \Cref{prop:tgt_err_clean}.
    Then for any $\delta\in(0,1)$, with probability at least $1-\delta$, empirical \pmdmean\ yields $\widehat{\pi}_{t+1}$ such that 
    \begin{align*}
        \mathcal{L}_t(\widehat{\pi}_{t+1})
        \lesssim
        \frac{\log(\abs{\Pi}/\delta)}{\tau^2n} +\epsilon_{\mathrm{opt}} +\frac{p_t\cdot\varepsilon_n(p_t,\delta) + \varepsilon_n(p_t,\delta)^2}{\tau^2}.
    \end{align*}
\end{lemma}
The complete proof is provided in Appendix~\ref{sec:refined}. 
\Cref{lem:pmdmean_refined} shows that \pmdmean\ now shares the same $O\bigl(\frac{\log(\abs{\Pi}/\delta)}{\tau^2n}\bigr)$ term as \pmdpart\ in \eqref{eq:erm_master_bound} (as $B_t^\mathrm{part}=O\bigl(\frac{1}{\tau}\bigr)$), while the $\overline{\Delta^2}$ term now becomes $O\bigl(\frac{p_t}{\tau^2}\sqrt{\frac{p_t}{n}}+\frac{1}{\tau^2n^2}\bigr)$, which vanishes as $n\to \infty$. 
This is better than the $O\bigl(\frac{1}{\tau^2}\bigr)$ error of \pmdpart\ in the small $n$ and $p_t$ regime. 
On the other hand, when $n=\Omega\bigl(\frac{1}{p_t}\bigr)$, the error of \pmdpart\ is $O(1)$, while the error in \pmdmean\ is $O\bigl(\frac{p_t^2}{\tau^2}\bigr)$. 
This suggests an adaptive regularization scheme of \pmdmean\ that scales $\tau$ with per-prompt pass rate $p_t$, which we leave for future investigation.

% ======================================================================
\section{Experiments}
\label{sec:experiments}

We conduct experiments on math reasoning RL to validate the practical performance of PMD. 
Our implementation is based on verl \citep{sheng2025hybridflow}. 

We train on the DAPO-Math-17k dataset \citep{yu2025dapo} with base models Qwen2.5-7B \citep{qwen2.5} and Qwen3-30B-A3B-Base \citep{yang2025qwen3}. 
The 7B models are trained for 495 global steps (15 epochs), while the 30B models are trained for $300$ global steps. 
We apply the same prompt formatting as in \citet{yu2025dapo}. 
The reward is binary $\pm 1$ based on the answer correctness only. 

We set the global batch size as 512 prompts with group size 16 and sampling temperature 1 for rollout. 
The maximum response length is 8192 for Qwen2.5-7B and 20480 for Qwen3-30B-A3B-Base. 
We train with a mini-batch size of 32 prompts (512 sequences) and learning rate $1\times 10^{-6}$.

We evaluate on AIME 2024 and AIME 2025. 
For each problem, we sample 32 solutions and report the average score. 
We mainly use GRPO \citep{shao2024deepseekmath} as the baseline. 
For efficiency comparison, we also include the on-policy gradient by setting the global batch size (rollout prompts) as 32 so that it equals the mini-batch size. 
More implementation details are provided in \Cref{sec:appendix-experiment}. 

\subsection{Main Results}
The main results are shown in \Cref{tab:main-results}. 
\pmdmean\ significantly outperforms the GRPO baseline: $\tau=0.005$ achieves +2.6\% AIME24 and +9.0\% AIME25 absolute gains on 7B model, and $\tau=0.1$ achieves +14.6\% AIME24 and +8.1\% AIME25 absolute gains on 30B MoE model. 

\begin{table}[t]
% \vspace{-0.1in}
    \caption{Overall evaluation scores (Avg@32). Staleness indicates the number of ministeps using rollouts from the same old policy.}
    % \vspace{-0.15in}
    \label{tab:main-results}
    \begin{center}\small
        % \begin{footnotesize}
            % \begin{sc}
                \begin{tabular}{lcccc}
                \toprule
                \textbf{Method ($\tau$)} & \textbf{Staleness} & \textbf{AIME 24} & \textbf{AIME 25} & \textbf{Average} \\
                \midrule
                \multicolumn{5}{c}{\textbf{Qwen2.5-7B}} \\
                \midrule
                GRPO (-) & 16 & 17.08 & 10.52 & 13.80 \\
                On-policy (-) & 1 & 18.65 & \underline{18.33} & \underline{18.49} \\
                \pmdmean\ (0.005) & 16 & \underline{19.69} & \textbf{19.48} & \textbf{19.58} \\
                \pmdmean\ (0.01) & 16 & 17.60 & 17.50 & 17.55 \\
                \pmdmean\ (0.02) & 16 & \textbf{22.50} & 16.67 & \textbf{19.58} \\
                \midrule
                \multicolumn{5}{c}{\textbf{Qwen3-30B-A3B-Base}} \\
                \midrule
                GRPO (-) & 16 & 36.56 & 27.92 & 32.24 \\
                \pmdmean\ (0.01) & 16 & \underline{50.00} & \underline{35.10} & \underline{42.55} \\
                \pmdmean\ (0.1) & 16 & \textbf{50.83} & \textbf{37.19} & \textbf{44.01} \\
                \bottomrule
                \end{tabular}
            % \end{sc}
        % \end{footnotesize}
    \end{center}
    % \vspace{-0.15in}
\end{table}

\begin{table}[tb]
    % \vspace{-0.1in}
    \caption{Comparing efficiency of on-policy gradient and \pmdmean. Timing is in milliseconds per token. While the actor update cost is comparable, a larger global batch size (high staleness) amortizes the inference cost and reduces overall training time.}
    \label{tab:efficiency_comparison}
    % \vspace{-0.15in}
    \begin{center}\small
        % \begin{sc}
            % \begin{scriptsize}
                \begin{tabular}{lccc}
                \toprule
                \textbf{Method} & \textbf{Overall (ms/token)} & \textbf{Generation} & \textbf{Actor Update} \\
                \midrule
                On-policy & 0.0569 & 0.0512 & \textbf{0.0057} \\
                \pmdmean & \textbf{0.0126} & \textbf{0.0062} & 0.0064 \\
                \bottomrule
                \end{tabular}
            % \end{scriptsize}
        % \end{sc}
    \end{center}
    % \vspace{-0.15in}
\end{table}

\textbf{Efficiency from larger global batch size.}
As shown in \Cref{tab:main-results,tab:efficiency_comparison}, compared to the on-policy gradient with staleness 1, the off-policy \pmdmean\ achieves comparable performance with $4.6\times$ speedup by leveraging a larger global rollout batch size that amortizes the inference cost.

\textbf{Stability.}
As shown in \Cref{fig:main_result}, \pmdmean\ remains stable during training, while \pmdpart\ is highly unstable and could collapse even with a much larger $\tau$. 

\begin{figure}[htb!]
    \vspace{-0.1in}
    \begin{center}
    
    \centerline{\includegraphics[width=0.5\columnwidth]{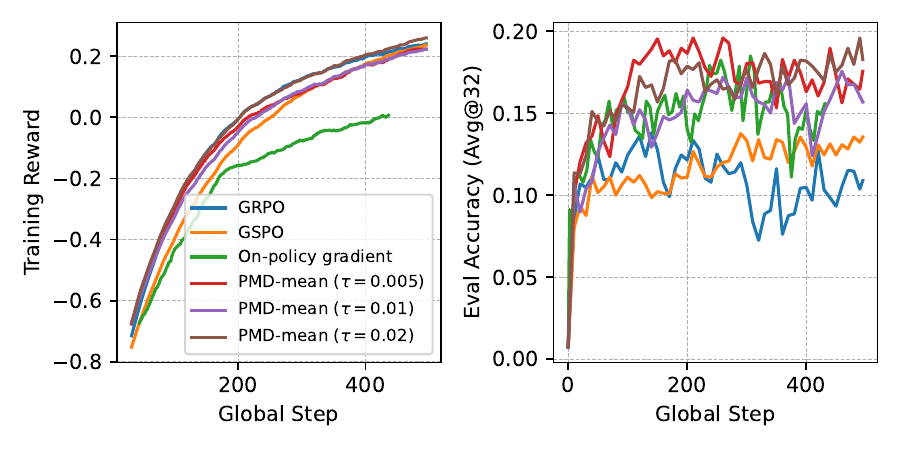}
    \includegraphics[width=0.5\columnwidth]{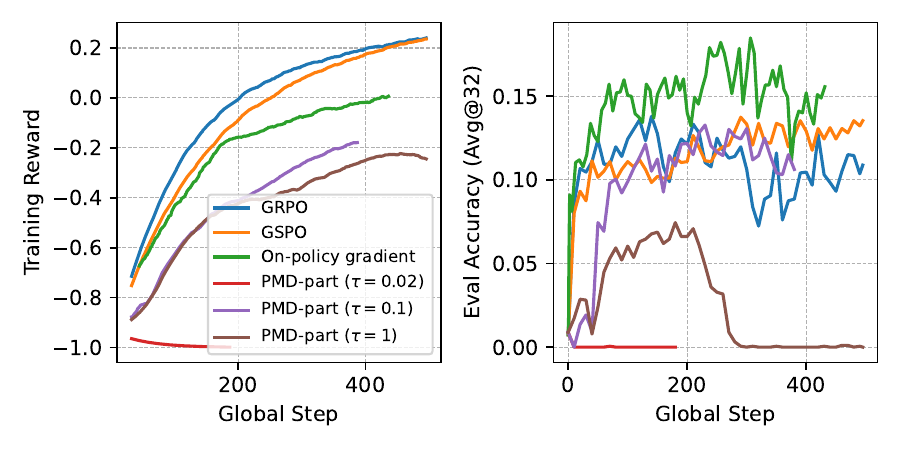}}
    % \vspace{-0.1in}
    % \centerline{\includegraphics[width=0.5\columnwidth]{figs/fig6.pdf}}
    % \vspace{-0.15in}
    \caption{
        Training curves (smoothed) of \pmdmean\ (upper) and \pmdpart\ (lower) with baselines for Qwen2.5-7B on DAPO-Math-17k (left) and the averaged evaluation accuracy on AIME 2024 and AIME 2025 (right). 
        The global step of on-policy gradient is divided by 16 to match other algorithms. 
    }
    \label{fig:main_result}
    \end{center}
    \vspace{-0.15in}
\end{figure}

\textbf{Policy ratios.}
We record the log policy ratios between the actor policy $\pi_\theta$ in the last mini-step and the old rollout policy $\pi_t$ of that global step, using it as an approximation of $\log\frac{\pi_{t+1}}{\pi_t}$. 
As shown in \Cref{fig:min_ratio}, the trend validates our theory in \Cref{sec:BB_instantiation} that the policy decrease in \pmdmean\ is weaker than \pmdpart, and becomes stronger as training proceeds and accuracy improves. 

\begin{figure}[tb]
      % \vspace{-0.1in}
  \begin{center}
    \centerline{\includegraphics[width=0.6\columnwidth]{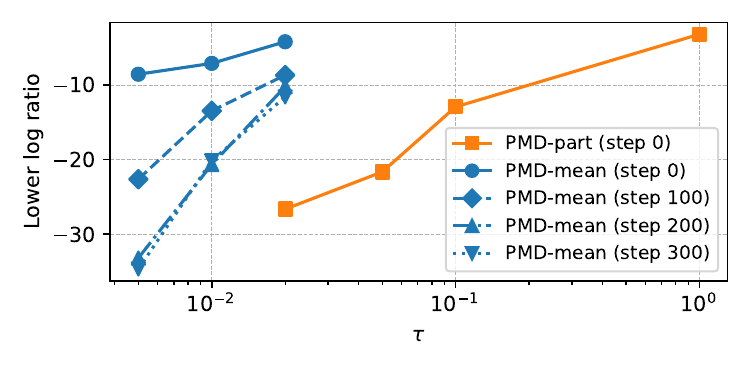}}
    \vspace{-0.15in}
    \caption{
        The minimum of log-ratios $\log\frac{\pi_{t+1}}{\pi_t}$ in \pmdmean\ and \pmdpart, estimated from the last update mini-batch. 
    }
    \label{fig:min_ratio}
  \end{center}
  \vspace{-0.25in}
\end{figure}

% \subsection{Comparison beyond Standard GRPO}
\textbf{Beyond standard GRPO.} 
Standard GRPO faces stability issues in training large MoE models. 
We further compare \pmdmean\ with GSPO \citep{zheng2025group}, which is an advanced variant of GRPO that incorporates sequence-level importance sampling (IS) with clipping and geometric mean normalization to resolve this MoE stability issue and achieves superior performance to GRPO. 
The results are shown in \Cref{tab:gspo}. 
As shown, \pmdmean\ outperforms GSPO on the Qwen2.5-7B and achieves comparable performance on the Qwen3-30B-A3B-Base MoE model. 
More detailed results are provided in \Cref{sec:appendix-experiments-supp}.

\begin{table}[htb!]
    % \vspace{-0.1in}
    \caption{Comparison of \pmdmean\ and GSPO (Avg@32). }
    \label{tab:gspo}
    % \vspace{-0.15in}
    \begin{center}\small
        % \begin{scriptsize}
            % \begin{sc}
                \begin{tabular}{lcccc}
                \toprule
                \textbf{Method} & \textbf{Model} & \textbf{AIME 24} & \textbf{AIME 25} & \textbf{Average} \\
                \midrule
                GSPO & 7B & 15.52 & 11.98 & 13.75 \\
                \pmdmean\ & 7B & \textbf{19.69} & \textbf{19.48} & \textbf{19.58} \\
                \midrule
                GSPO & 30B & \textbf{53.33} & 34.58 & {43.96} \\
                \pmdmean\ & 30B & {50.83} & \textbf{37.19} & \textbf{44.01} \\
                \bottomrule
                \end{tabular}
            % \end{sc}
        % \end{scriptsize}
    \end{center}
    % \vspace{-0.3in}
\end{table}

% ======================================================================
\section{Related Work}
\label{sec:related}

\textbf{RL for Post-Training of LLMs.}
Contemporary LLM post-training and alignment frameworks predominantly utilize reinforcement learning with human or AI feedback (RLHF/RLAIF, \citealt{ziegler2019fine,ouyang2022training,bai2022constitutional}) or verifiable rewards (RLVR, \citealt{lambert2024tulu}).
This paradigm has demonstrated particular efficacy for mathematical reasoning, coding, and logical tasks, subsequently inspiring large-scale RL methodologies and architectural designs for increasingly complex agentic capabilities \citep{jaech2024openai,guo2025deepseek,google2025gemini,team2025kimik20}.

Policy gradient methods \citep{williams1992simple,sutton1999policy}, especially TRPO and PPO \citep{schulman2015trust,schulman2017proximal}, have established themselves as foundational approaches in reinforcement learning. 
However, within LLM post-training contexts, maintaining parameterized value networks (critic models) introduces substantial estimation biases and computational overhead. 
To mitigate these limitations, recent methods such as GRPO \citep{shao2024deepseekmath} and RLOO \citep{ahmadian2024back} eliminate the critic component and instead leverage group-based relative baselines estimated from multiple Monte Carlo samples per prompt. 
More sophisticated approaches, such as DAPO \citep{yu2025dapo}, further refine performance through advanced optimization techniques.

These methods fundamentally depend on sampling distributions that closely match the current policy distribution, necessitating off-policy correction mechanisms to maintain training stability. 
While PPO/GRPO implement token-level importance sampling (IS) with clipping, more recent algorithms such as GSPO \citep{zheng2025group} and CISPO \citep{chen2025minimax} employ sequence-level IS or detached clipping IS to enhance stability when training mixture-of-expert (MoE) models. 
Additional research addresses the training-inference distribution mismatch at the infrastructure level, proposing methodological refinements including truncated IS \citep{yao2025offpolicy} and masked IS \citep{liu-li-2025-rl-collapse}. 
Although partially effective, these techniques substantially increase algorithmic complexity and incorporate numerous empirical adjustments that resist theoretical analysis. 
In contrast, our investigation focuses on the minimalist PMD algorithm, which offers greater analytical transparency while delivering competitive empirical performance.

\textbf{Policy Mirror Descent.}
The mirror descent framework \citep{nemirovski1983problem} provides a classical formulation for policy optimization in reinforcement learning \citep{geist2019theory,tomarmirror}, with extensive literature analyzing its theoretical iteration and sample complexities \citep{xiao2022convergence,zhan2023policy,lan2023policy,yuan2023linear,alfano2023novel,xu2024sample}. 
However, these analyses predominantly address tabular settings or function approximation scenarios with abundant samples, rather than the practical constraints of LLM post-training where rollouts are necessarily limited. 
Our work establishes a novel connection between practical LLM post-training methodologies and the choice of Bregman divergence, demonstrating that \pmdmean\ implicitly optimizes an adaptive mixed KL--$\chi^2$ regularizer. 
This mixed regularization approach shares conceptual similarities with $\chi^2$PO \citep{huang2024correcting}, which employs mixed KL--$\chi^2$ divergence to mitigate overfitting in KL-regularized DPO \citep{rafailov2023direct} under distribution shift conditions. 
While alternative regularization schemes have been explored in offline preference learning contexts \citep{wang2023beyond}, our investigation specifically addresses online policy optimization challenges.

The literature offers various regression-based approaches to approximate the ideal KL solution in \pmdpart. 
\citet{richemond2024offline} propose incorporating a value network to estimate the log-partition function, which dates back to \citet{nachum2017bridging}. 
\citet{gao2024rebel} develop a technique for fitting pairwise relative rewards that eliminates the partition term entirely. 
\citet{bartoldson2025trajectory} approximate the log-partition term in the loss using the group average of all other terms (with stop-grad). 
In contrast, \pmdmean\ \citep{team2025kimik15} implements a simpler strategy by directly approximating the log-partition function with the mean reward. 
Our analysis focuses on characterizing the theoretical properties of this practical approximation and establishing its mathematical foundations.

% ======================================================================
\section{Conclusion}
\label{sec:conclusion}
This paper presents a comprehensive analysis of \pmdmean, a practical algorithm for large-scale LLM post-training that has been deployed in leading language models. 
The analysis provides an exact characterization of the algorithm's population update through the Lambert-$W$ function, establishing its mathematical equivalence to solving mirror descent subproblems with an adaptive mixed KL--$\chi^2$ regularizer. 
This theoretical framework reveals a concrete mechanism underlying the algorithm's stability: the induced $\chi^2$ term systematically constrains large probability ratio changes, effectively preventing overly aggressive policy updates that often lead to training instability.

The investigation deliberately focuses on the principled form of \pmdmean\ to enable clearer theoretical analysis and understanding. 
Advanced techniques such as oversampling strategies and importance sampling corrections for addressing training/inference engine mismatches could potentially enhance \pmdmean's performance further, and these directions are reserved for future research. 
By elucidating the fundamental mathematical properties of \pmdmean, this work contributes to the development of theoretically grounded yet practically effective RL algorithms for LLM post-training, potentially enabling simpler, more robust, and scalable approaches to this increasingly critical task.

\bibliography{ref}
\bibliographystyle{bibstyle}

%%%%%%%%%%%%%%%%%%%%%%%%%%%%%%%%%%%%%%%%%%%%%%%%%%%%%%%%%%%%%%%%%%%%%%%%%%%%%%%
%%%%%%%%%%%%%%%%%%%%%%%%%%%%%%%%%%%%%%%%%%%%%%%%%%%%%%%%%%%%%%%%%%%%%%%%%%%%%%%
% APPENDIX
%%%%%%%%%%%%%%%%%%%%%%%%%%%%%%%%%%%%%%%%%%%%%%%%%%%%%%%%%%%%%%%%%%%%%%%%%%%%%%%
%%%%%%%%%%%%%%%%%%%%%%%%%%%%%%%%%%%%%%%%%%%%%%%%%%%%%%%%%%%%%%%%%%%%%%%%%%%%%%%
\newpage
\appendix
\onecolumn

\section{Experimental Details}
\label{sec:appendix-experiment}
We train on the DAPO-Math-17k\footnote{\url{https://huggingface.co/datasets/BytedTsinghua-SIA/DAPO-Math-17k}} dataset (deduplicated). 
The base models include Qwen2.5-7B\footnote{\url{https://huggingface.co/Qwen/Qwen2.5-7B}} and Qwen3-30B-A3B-Base\footnote{\url{https://huggingface.co/Qwen/Qwen3-30B-A3B-Base}}. 

\subsection{Prompt Template}
Our prompt template follows \citet{yu2025dapo}, with all questions \texttt{problem\_statement} processed in the following form. 
\begin{tcblisting}{listing only, title=Chain-of-Thought (CoT) Prompt Template}
Solve the following math problem step by step. The last line of your response should be of the form Answer: $Answer (without quotes) where $Answer is the answer to the problem.

{problem_statement}

Remember to put your answer on its own line after "Answer:".
\end{tcblisting}

\subsection{Hyperparameters}
We summarize key hyperparameters in \Cref{tab:hyperparams}. 
\begin{table}[htb]
    \centering
    \caption{Key hyperparameters for 7B dense model and 30B MoE model experiments.}
    \label{tab:hyperparams}
    \small
    \begin{tabular}{lcc}
        \toprule
        \textbf{Parameter} & \textbf{7B Dense} & \textbf{30B MoE} \\
        \midrule
        trainer.nnodes & 4 & 8 \\
        trainer.n\_gpus\_per\_node & 8 & 8 \\
        distributed strategy & FSDP & Megatron \\
        model.path & Qwen/Qwen2.5-7B & Qwen/Qwen3-30B-A3B-Base \\
        \midrule
        % data.train\_files & dapo-math-17k.parquet & dapo-math-17k.parquet \\
        % data.val\_files & AIME24 + AIME25 & AIME24 + AIME25 \\
        data.train\_batch\_size (prompts) & 512 & 512 \\
        data.gen\_batch\_size (prompts) & 512 & 512 \\
        data.max\_prompt\_length & 2048 & 2048 \\
        data.max\_response\_length & 8192 & 20480 \\
        \midrule
        rollout.name & vLLM & vLLM \\
        rollout.n (group size) & 16 & 16 \\
        rollout.temperature & 1.0 & 1.0 \\
        rollout.top\_p & 1.0 & 1.0 \\
        rollout.max\_model\_len & 10240 & 22528 \\
        val\_kwargs.n (avg@k) & 32 & 32 \\
        val\_kwargs.temperature & 1.0 & 1.0 \\
        val\_kwargs.top\_p & 0.7 & 0.7 \\
        % \midrule
        % algorithm.adv\_estimator & RLOO & RLOO \\
        % policy\_loss.loss\_mode & opmd & opmd \\
        % policy\_loss.pmd\_tau & 0.005 & 0.1 \\
        % algorithm.partition\_tau & 0.005 & 0.1 \\
        \midrule
        actor.ppo\_epochs & 1 & 1 \\
        actor.ppo\_mini\_batch\_size (prompts) & 32 & 32 \\
        actor.clip\_ratio\_low / high (GRPO) & 0.2 / 0.2 & 0.2 / 0.2 \\
        actor.clip\_ratio\_low / high (GSPO) & 3e-4 / 4e-4 & 3e-4 / 4e-4 \\
        optim.lr & 1e-6 & 1e-6 \\
        optim.betas & [0.9, 0.999] & [0.9, 0.999] \\
        optim.weight\_decay & 0.01 & 0.01 \\
        grad clip & 1.0 & 1.0 \\
        \midrule
        tensor\_model\_parallel\_size & 1 & 2 \\
        pipeline\_model\_parallel\_size & 1 & 2 \\
        expert\_model\_parallel\_size & N/A & 8 \\
        \bottomrule
    \end{tabular}
\end{table}

\subsection{Implementation Details}
Our implementation is based on verl\footnote{\url{https://github.com/verl-project/verl}} \citep{sheng2025hybridflow}. 
For GRPO and GSPO, we disable explicit KL penalty with respect to the base reference model, following the DAPO recipe \citep{yu2025dapo}. 
For each prompt $x$, we sample $K$ responses $y_1,\dots,y_K\sim \pi_t(\cdot\mid x)$ from the old policy $\pi_t(\cdot\mid x)\coloneqq\pi_{\theta_t}(\cdot\mid x)$, and get rewards $r_i\coloneqq r(x,y_i)\in\{-1,+1\}$ based on the correctness of answers. 
Let $\pi_\theta$ be the trainable policy. We have $\log \pi_\theta(y\mid x)= \sum_{j=1}^{\abs{y}}\log \pi_\theta(y_j\mid x,y_{<j})$. 

Define the token probability ratio and geometric normalized sequence probability ratio as follows: 
\begin{align*}
    \rho_{i,j}(\theta)\coloneqq \frac{\pi_\theta(y_{i,j}\mid x,y_{i,<j})}{\pi_t(y_{i,j}\mid x,y_{i,<j})}, 
    \quad
    s_{i}(\theta)\coloneqq \left(\frac{\pi_\theta(y_i\mid x)}{\pi_t(y_i\mid x)}\right)^{\frac{1}{\abs{y_i}}}. 
\end{align*}
Moreover, define the following advantages: 
\begin{align*}
    \widehat{A}^{\mathrm{grpo}}_i\coloneqq \frac{r_i-\mathrm{mean}(r_1,\dots,r_K)}{\mathrm{std}(r_1,\dots,r_K)}, 
    \quad
    \widehat{A}^{\mathrm{loo}}_i\coloneqq r_i-\frac{1}{K-1}\sum_{j\neq i} r_j,
    \quad
    \widehat{A}^{\mathrm{part}}_i\coloneqq r_i-\tau\log\left(\frac{1}{K-1}\sum_{j\neq i} e^{r_j/\tau}\right).
\end{align*}

The GRPO \citep{shao2024deepseekmath} loss is defined as
\begin{align*}
    \mathcal{L}_\mathrm{GRPO}(\theta)=
    -\mathbb{E}_{x\sim \mathcal{D}}\mathbb{E}_{y_1,\dots,y_K\sim\pi_t(\cdot\mid x)}\left[\frac{1}{K}\sum_{i=1}^K\frac{1}{\abs{y_i}}\sum_{j=1}^{\abs{y_i}} \min\Big(\rho_{i,j}(\theta)\widehat{A}^{\mathrm{grpo}}_i, \mathrm{clip}(\rho_{i,j}(\theta),1-\epsilon,1+\epsilon)\widehat{A}^{\mathrm{grpo}}_i\Big)\right],
\end{align*}
where $\epsilon=0.2$ and we discard the KL penalty term. 

When the global batch size is set to be the same as the mini-batch size, GRPO reduces to the on-policy gradient with advantage estimator $\widehat{A}^\mathrm{grpo}_i$. 
Empirically, we find that using $\widehat{A}^\mathrm{loo}_i$ yields similar performance, and thus we use the (length-normalized) RLOO \citep{ahmadian2024back} loss for on-policy gradient experiments. 
\begin{align*}
    \mathcal{L}_{\mathrm{RLOO}}(\theta) = -\mathbb{E}_{x\sim \mathcal{D}}\mathbb{E}_{y_1,\dots,y_K\sim\pi_t(\cdot\mid x)} \left[\frac{1}{K}\sum_{i=1}^K\frac{1}{\abs{y_i}}\widehat{A}^{\mathrm{loo}}_i \log \pi_\theta(y_i\mid x)\right].
\end{align*}

The GSPO \citep{zheng2025group} loss is defined as follows: 
\begin{align*}
    \mathcal{L}_{\mathrm{GSPO}}(\theta) = -\mathbb{E}_{x\sim \mathcal{D}}\mathbb{E}_{y_1,\dots,y_K\sim\pi_t(\cdot\mid x)}\left[\frac{1}{K}\sum_{i=1}^K \min\Big(s_{i}(\theta)\widehat{A}^{\mathrm{grpo}}_i, \mathrm{clip}(s_i(\theta),1-\epsilon_\mathrm{low},1+\epsilon_\mathrm{high})\widehat{A}^{\mathrm{grpo}}_i\Big)\right],
\end{align*}
where $\epsilon_\mathrm{low}=3\times 10^{-4}$ and $\epsilon_\mathrm{high}=4\times 10^{-4}$ as suggested. 

We implement \pmdmean\ \citep{team2025kimik15} and \pmdpart\ using the following losses. 
\begin{align*}
    \mathcal{L}_{\mathrm{mean}}(\theta) = \mathbb{E}_{x\sim \mathcal{D}}\mathbb{E}_{y_1,\dots,y_K\sim\pi_t(\cdot\mid x)} \left[\frac{1}{K}\sum_{i=1}^K\frac{\tau}{\abs{y_i}}\left(\log\frac{\pi_\theta(y_i\mid x)}{\pi_t(y_i\mid x)}-\frac{\widehat{A}^{\mathrm{loo}}_i}{\tau}\right)^2 \right],\\
    \mathcal{L}_{\mathrm{part}}(\theta) = \mathbb{E}_{x\sim \mathcal{D}}\mathbb{E}_{y_1,\dots,y_K\sim\pi_t(\cdot\mid x)} \left[\frac{1}{K}\sum_{i=1}^K\frac{\tau}{\abs{y_i}}\left(\log\frac{\pi_\theta(y_i\mid x)}{\pi_t(y_i\mid x)}-\frac{\widehat{A}^{\mathrm{part}}_i}{\tau}\right)^2 \right].
\end{align*}
The factor $\tau$ in the loss ensures the gradient norm does not differ too much when tuning $\tau$, and length normalization is consistent with the loss aggregation mode in other methods. 

\subsection{Supplementary Results}\label{sec:appendix-experiments-supp}
We provide supplementary experimental results in \Cref{fig:results-7B,fig:results-30B}. 

\begin{figure}[ht]
  % \vskip 0.2in
  \begin{center}
    \centerline{\includegraphics[width=0.75\columnwidth]{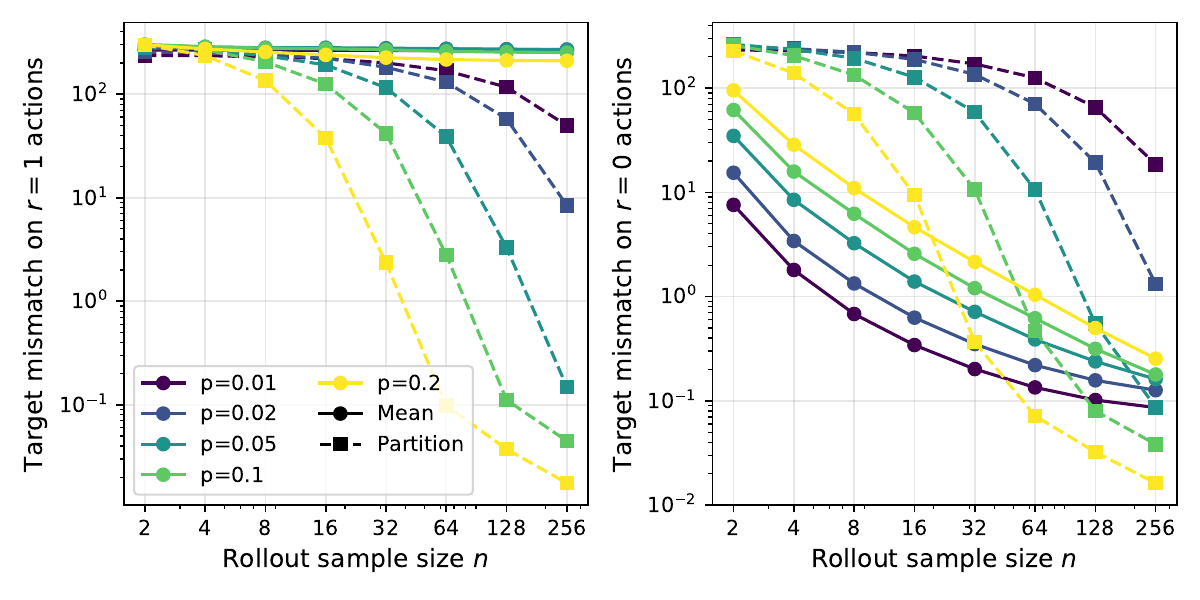}}
    \caption{
      Target estimation error for positive and negative actions in \pmdmean\ and \pmdpart\ under $\tau=0.05$ and $p_t$ ranges from $0.01$ to $0.2$. 
      Left: positive actions. Right: negative actions. 
      The plot shows the average from $100$ random seeds. 
      The error in \pmdmean\ mainly comes from a systematic mismatch between positive targets. 
    }
    \label{fig:target_actions}
  \end{center}
\end{figure}

\begin{figure}[ht]
  % \vskip 0.2in
  \begin{center}
    \centerline{\includegraphics[width=\columnwidth]{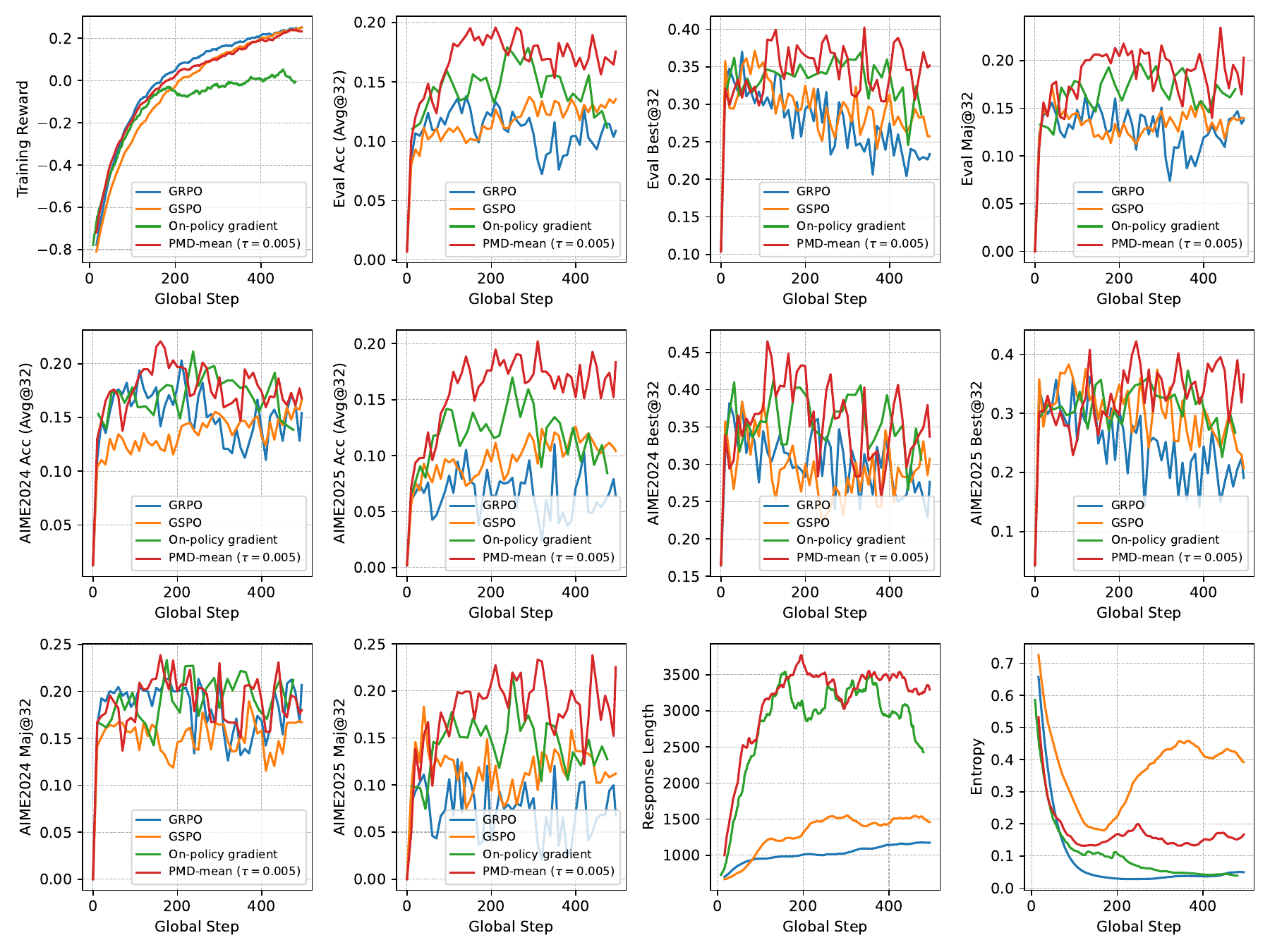}}
    \caption{
      Qwen2.5-7B training results for 15 epochs (495 global steps). Training reward, response length, and entropy are smoothed by EMA with an effective window size of 50. 
      \pmdmean\ achieves superior performance not only in Pass@1 (measured in Avg@32) but also Pass@32 and Maj@32 (accuracy of majority voting answer). 
    }
    \label{fig:results-7B}
  \end{center}
\end{figure}

\begin{figure}[ht]
  % \vskip 0.2in
  \begin{center}
    \centerline{\includegraphics[width=\columnwidth]{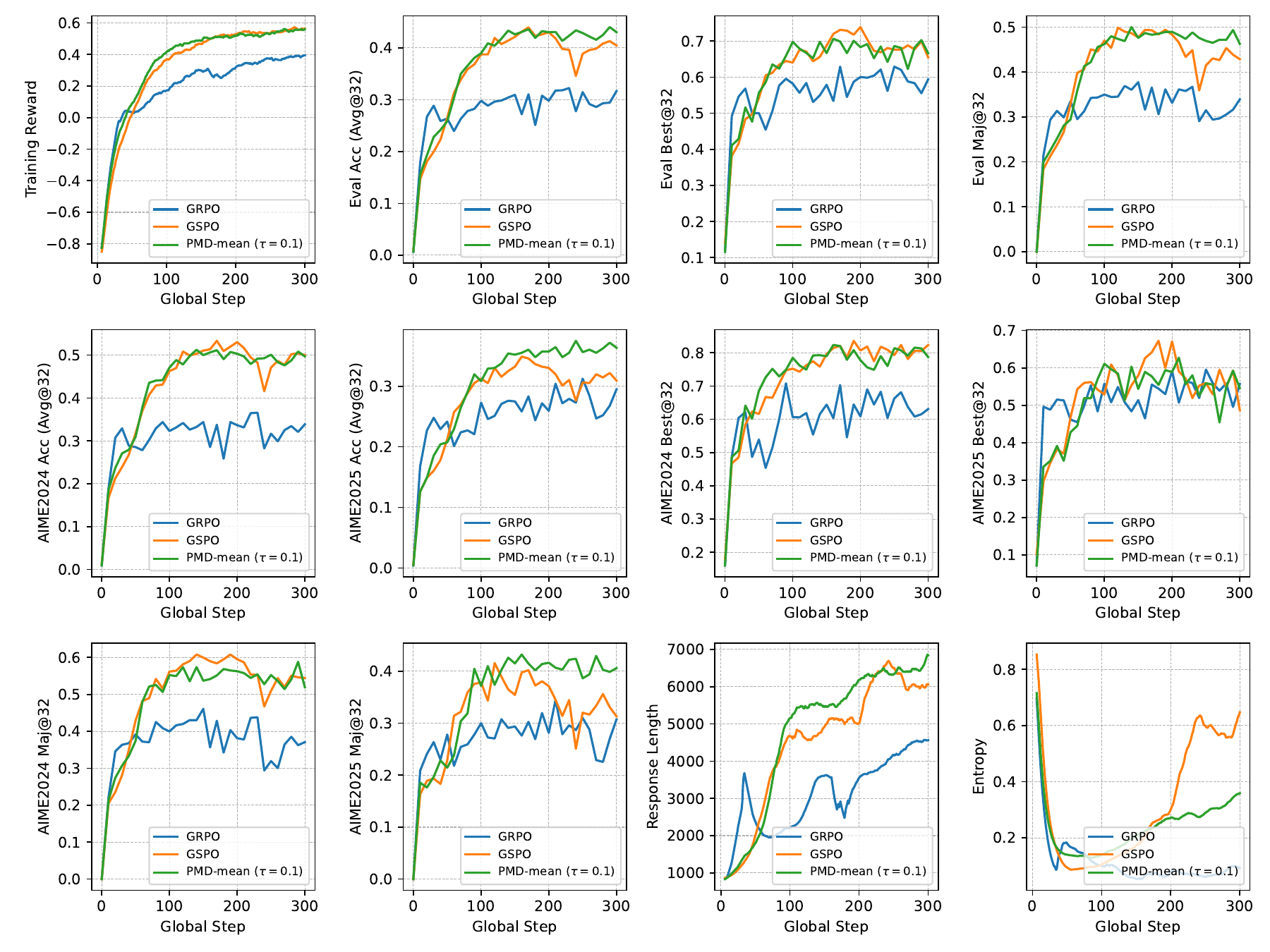}}
    \caption{
      Qwen3-30B-A3B-Base training results for 300 global steps. Training reward, response length, and entropy are smoothed by EMA with an effective window size of 20. 
    }
    \label{fig:results-30B}
  \end{center}
\end{figure}

\section{Missing Proofs in Section~\ref{sec:pmdmean}}
\subsection{Proof of \Cref{thm:pmdmean_solution}}
\label{sec:proof_thm1}
\begin{proof}[Proof of Theorem~\ref{thm:pmdmean_solution}]
    Define $u(y)=\log\frac{\pi(y)}{\pi_t(y)}$, then the Lagrangian of $\mathcal{L}_{\mathrm{mean}}$ (in the policy space) is written as
    \begin{align*}
        L(u,\lambda)=\frac{1}{2}\sum_{y\in\mathcal{Y}}\pi_t(y)(\Delta_y-\tau u(y))^2 + \lambda\cdot\left(\sum_{y\in\mathcal{Y}}\pi_t(y)e^{u(y)}-1\right),
    \end{align*}
    The KKT conditions yield 
    \begin{align*}
        \begin{cases}-\tau\pi_t(y)(\Delta_y-\tau u(y)) + \lambda\pi_t(y)e^{u(y)}=0,~\forall y\in\mathcal{Y},\\
        \sum_{y\in\mathcal{Y}}\pi_t(y)e^{u(y)}=1.
        \end{cases}
    \end{align*}
    Assume $\pi_t(y)>0$ for all $y\in\mathcal{Y}$ (which is true for LLMs without top-p/top-k constraints), then any stationary point of $\mathcal{L}_{\mathrm{mean}}$ should satisfy 
    \begin{align}
        -\tau(\Delta_y-\tau u(y))+\lambda e^{u(y)}=0
        &\iff \tau(\Delta_y-\tau u(y))e^{-u(y)}=\lambda\notag\\
        &\iff \left(\frac{\Delta_y}{\tau}-u(y)\right)e^{\frac{\Delta_y}{\tau}-u(y)}=\frac{\lambda}{\tau^2}e^{\frac{\Delta_y}{\tau}}\notag\\
        &\iff \frac{\Delta_y}{\tau}-u(y)=W\left(\frac{\lambda}{\tau^2}e^{\frac{\Delta_y}{\tau}}\right)\notag\\
        &\iff u(y)=\frac{\Delta_y}{\tau}-W\left(\frac{\lambda}{\tau^2}e^{\frac{\Delta_y}{\tau}}\right).\label{eq:KKT-mean}
    \end{align}
    The RHS of \eqref{eq:KKT-mean} is monotonically decreasing in $\lambda$, and $\mathbb{E}_{\pi_t}[\text{RHS}]=0$ when $\lambda=0$. Moreover, by Jensen’s inequality, 
    \begin{align*}
    1=\mathbb{E}_{y\sim\pi_t}[e^{u(y)}]\geq e^{\mathbb{E}_{y\sim\pi_t}[u(y)]}\implies\mathbb{E}_{y\sim\pi_t}[u(y)]\leq 0,
    \end{align*}
    thus there must be a unique $\lambda\geq 0$ (and hence $u$) such that the KKT conditions are satisfied. Moreover, the Hessian of Lagrangian in $u$ is positive definite under the assumption that $\pi_t(y)>0$, thus the point is indeed the minimizer of $\mathcal{L}_{\mathrm{mean}}$. 
    
    To characterize $\lambda$, we invoke several properties of $W$ for all $z\geq 0$: (1) $\cW(z)$ is concave and monotonically increasing; (2) $\cW(z)\geq\frac{z}{1+z}$; (3) $e^{\cW(z)}=\frac{z}{\cW(z)}$.
    Moreover, by the feasibility constraint, we have 
    \begin{align*}
        1
        &=\sum_{y\in\mathcal{Y}}\pi_t(y)\frac{\exp\left(\frac{\Delta_y}{\tau}\right)}{\exp\left(W\left( \frac{\lambda}{\tau^2} \exp\left(\frac{\Delta_y}{\tau}\right)\right)\right)}\\
        &=\sum_{y\in\mathcal{Y}}\pi_t(y)\frac{\exp\left(\frac{\Delta_y}{\tau}\right)W\left(\frac{\lambda}{\tau^2} \exp\left(\frac{\Delta_y}{\tau}\right)\right)}{\frac{\lambda}{\tau^2} \exp\left(\frac{\Delta_y}{\tau}\right)}\\
        &=\frac{\tau^2}{\lambda}\mathbb{E}_{y\sim\pi_t}\left[W\left(\frac{\lambda}{\tau^2} \exp\left(\frac{\Delta_y}{\tau}\right)\right)\right].
    \end{align*}
    For convenience, let $x\coloneqq\lambda/\tau^2\geq 0$, $A\coloneqq\mathbb{E}_{\pi_t}[e^{\Delta_y/\tau}]$, $B\coloneqq\mathbb{E}_{\pi_t}[e^{2\Delta_y/\tau}]$. Then the target is to show 
    \begin{align*}
        \frac{A(A-1)}{B}\leq x\leq\log A.
    \end{align*}
    
    For the upper bound, by concavity and Jensen’s inequality, 
    \begin{align*}
        x=\mathbb{E}_{\pi_t}[\cW(xe^{\Delta_y/\tau})]
        \leq \cW(\mathbb{E}[xe^{\Delta_y/\tau}])
        =\cW(xA).
    \end{align*}
    Since $W$ is increasing, above inequality implies 
    \begin{align*}
        xe^x\leq \cW(xA)e^{\cW(xA)}=xA\implies e^x\leq A\implies x\leq\log A,
    \end{align*}
    thus proving the upper bound.
    
   On the other hand, by lower bound on $\cW$, 
    \begin{align*}
        x&=\mathbb{E}_{\pi_t}[\cW(xe^{\Delta_y/\tau})]\\
        &\geq\mathbb{E}_{\pi_t}\left[\frac{xe^{\Delta_y/\tau}}{1+xe^{\Delta_y/\tau}}\right]\\
        &\geq
        \frac{\left(\mathbb{E}_{\pi_t}[xe^{\Delta_y/\tau}]\right)^2}{\mathbb{E}_{\pi_t}[xe^{\Delta_y/\tau}(1+xe^{\Delta_y/\tau})]}\\
        &=\frac{(xA)^2}{xA+x^2B},
    \end{align*}
    where the second inequality is from Cauchy-Schwarz. Solving the inequality yields the lower bound.

    For binary rewards $r\in\{0,1\}$ with $p=\E_{\pi_t}[r(y)]$, we have $\E_{\pi_t}[\Delta_y^2]=\Var(r)=p(1-p)$ and
    \begin{align}\label{eq:expression-x}
        x=p \cW(xe^{(1-p)/\tau}) + (1-p)\cW(xe^{-p/\tau}).
    \end{align}
    For large $\tau\to\infty$, we have 
    \begin{align*}
        &A=\mathbb{E}_{\pi_t}\left[1+\frac{\Delta_y}{\tau} + \frac{\Delta_y^2}{2\tau^2} + O(\tau^{-3})\right]=1+\frac{p(1-p)}{2\tau^2} + O(\tau^{-3}),\\
        &B=\mathbb{E}_{\pi_t}\left[1+\frac{2\Delta_y}{\tau} + \frac{2\Delta_y^2}{\tau^2} + O(\tau^{-3})\right]=1+\frac{2p(1-p)}{\tau^2} + O(\tau^{-3}).
    \end{align*}
    Thus, the lower bounds on $x=\lambda/\tau^2$ yield
    \begin{align*}
        A(A-1)\leq xB
        &\implies\left(\frac{p(1-p)}{2\tau^2}\right)\left(1+ O(\tau^{-2})\right)\leq \left(1+O(\tau^{-2})\right)\frac{\lambda}{\tau^2}\\
        &\implies \frac{p(1-p)}{2}-O(\tau^{-1})\leq\lambda.
    \end{align*}
    On the other hand, the upper bound yields
    \begin{align*}
        \lambda
        &\leq\tau^2\log A\\
        &=\tau^2\log\left(1+\frac{p(1-p)}{2\tau^2}+O(\tau^{-3})\right)\\
        &\leq\tau^2\left(\frac{p(1-p)}{2\tau^2}+O(\tau^{-3})\right)\\
        &=\frac{p(1-p)}{2}+O(\tau^{-1}).
    \end{align*}
    Combine the two bounds, we have $\lambda=\frac{1}{2}p(1-p)+O(\tau^{-1})$ as $\tau\to\infty$. 

    For small $\tau\to 0$, the bounds \eqref{eq:lambda_bounds} are too loose. 
    We define $v=\frac{\lambda}{\tau}=\tau x$ and show that $v\to p(1-p)$. 
    Firstly, the Lambert-$W$ function satisfies 
    \begin{align*}
        \log z-\log\log z\leq \cW(z)\leq \log z
    \end{align*}
    for $z>e$. 
    Moreover, for $z\geq 0$, $\cW(z)=\frac{z}{e^{\cW(z)}}\leq z$. 
    Since $e^{-p/\tau}$ decays faster than any polynomial in $\tau^{-1}$, we have 
    \begin{align*}
        0\leq (1-p)\tau \cW(xe^{-p/\tau})
        \leq (1-p)\tau xe^{-p/\tau}
        =(1-p)ve^{-p/\tau}\to 0.
    \end{align*}
    Therefore, the second term in \eqref{eq:expression-x} vanishes. 
    For the remaining dominant term, let $z_1=xe^{(1-p)/\tau}=\frac{v}{\tau}e^{(1-p)/\tau}$, then $z_1\to\infty$ when $\tau\to 0$, and $\log z_1=\frac{1-p}{\tau}+\log\frac{v}{\tau}$. 
    In this case, our bound on the Lambert-$W$ function gives 
    \begin{align*}
        \tau \cW(z_1)
        =(1-p) + o(1),
    \end{align*}
    and thus 
    \begin{align*}
        v=\tau x
        =\tau\left(p \cW(z_1) + o(1)\right)
        =p(1-p) + o(1)
        \implies 
        \lambda=\tau p(1-p)(1+o(1)).
    \end{align*}
\end{proof}

\subsection{Policy Ratio}\label{sec:appendix-ratio}
We formally state the policy ratios of \pmdmean\ and \pmdpart\ in \Cref{eq:ratio_pos_mean,eq:ratio_pos_part,eq:ratio_neg_mean,eq:ratio_neg_part} in the following proposition. 
\begin{proposition}[Binary-reward ratios for small $\tau$]
\label{prop:binary_ratios}
    Assume $r(y)\in\{0,1\}$ and let $p=\E_{\pi_t}[r(y)]\in(0,1)$.
    Consider the ratios $\rho(y)\coloneqq\pi_{t+1}(y)/\pi_t(y)$.
    
    \noindent\textbf{\pmdmean.}
    As $\tau\to 0$, for any $y$ with $r(y)=1$,
    \begin{align*}
        \rho^{\mathrm{mean}}_+(y)
        =\frac{1}{p}-\frac{1-p}{p}e^{-p/\tau}(1+o(1)),
    \end{align*}
    and for any $y$ with $r(y)=0$,
    \begin{align*}
        \rho^{\mathrm{mean}}_-(y)=e^{-p/\tau}(1+o(1)).
    \end{align*}
    
    % \smallskip
    \noindent\textbf{\pmdpart.}
    For any $\tau>0$, the partition update satisfies
    \begin{align*}
        \rho^{\mathrm{part}}_+(y)=\frac{1}{p+(1-p)e^{-1/\tau}},
        \quad
        \rho^{\mathrm{part}}_-(y)=\frac{e^{-1/\tau}}{p+(1-p)e^{-1/\tau}},
    \end{align*}
    and in particular as $\tau\to 0$,
    \begin{align*}
        \rho^{\mathrm{part}}_+(y)=\frac{1}{p}-\frac{1-p}{p^2}e^{-1/\tau}+O(e^{-2/\tau}),
        \quad
        \rho^{\mathrm{part}}_-(y)=\frac{1}{p}e^{-1/\tau}+O(e^{-2/\tau}).
    \end{align*}
\end{proposition}
\begin{proof}[Proof of \Cref{prop:binary_ratios}]
    Throughout, let $x\coloneqq\lambda/\tau^2$.
    For \pmdmean, start from the Lambert-$W$ form \eqref{eq:pmdmean_lambertW_form} and use $e^{-\cW(z)}=\cW(z)/z$ to rewrite the ratio as
    \begin{align}\label{eq:ratio_W_form}
        \frac{\pi^{\mathrm{mean}}_{t+1}(y)}{\pi_t(y)}
        =\exp\Big(\frac{\Delta_y}{\tau}-\cW(xe^{\Delta_y/\tau})\Big)
        =\frac{1}{x}W\big(xe^{\Delta_y/\tau}\big).
    \end{align}
    In the binary case, $\Delta=1-p$ when $r=1$ and $\Delta=-p$ when $r=0$, so defining
    \begin{align*}
        \rho^{\mathrm{mean}}_+ \coloneqq \frac{1}{x}W\Big(xe^{(1-p)/\tau}\Big),
        \quad
        \rho^{\mathrm{mean}}_- \coloneqq \frac{1}{x}W\Big(xe^{-p/\tau}\Big),
    \end{align*}
    we have the normalization identity
    \begin{align}
        \label{eq:norm_binary}
        1=\sum_y\pi^{\mathrm{mean}}_{t+1}(y)=p\rho^{\mathrm{mean}}_+ + (1-p)\rho^{\mathrm{mean}}_-.
    \end{align}
    
    By \eqref{eq:lambda_asymptotics}, $\lambda\sim \tau p(1-p)$ as $\tau\to 0$, hence
    \begin{align*}
        x=\frac{\lambda}{\tau^2}\sim \frac{p(1-p)}{\tau}=\Theta\left(\frac{1}{\tau}\right).
    \end{align*}
    Therefore $xe^{-p/\tau}\to 0$ as $\tau\to 0$. Using the Taylor expansion
    $\cW(z)=z+O(z^2)$ as $z\to 0$, we obtain
    \begin{align*}
        \cW(xe^{-p/\tau}) = xe^{-p/\tau}(1+o(1)),
    \end{align*}
    and plugging into \eqref{eq:ratio_W_form} yields
    \begin{align*}
        \rho^{\mathrm{mean}}_- = \frac{1}{x}\cW(xe^{-p/\tau}) = e^{-p/\tau}(1+o(1)).
    \end{align*}
    
    For $\rho^{\mathrm{mean}}_+$, substituting the above into \eqref{eq:norm_binary} gives
    \begin{align*}
        \rho^{\mathrm{mean}}_+
        =\frac{1-(1-p)\rho^{\mathrm{mean}}_-}{p}
        =\frac{1}{p}-\frac{1-p}{p}e^{-p/\tau}(1+o(1)).
    \end{align*}
    This proves the \pmdmean\ claims.

    For \pmdpart, the update is explicit:
    \begin{align*}
        \pi^{\mathrm{part}}_{t+1}(y)
        =\frac{\pi_t(y)e^{r(y)/\tau}}{p e^{1/\tau}+(1-p)}.
    \end{align*}
    Hence
    \begin{align*}
        \rho^{\mathrm{part}}_+ = \frac{e^{1/\tau}}{p e^{1/\tau}+(1-p)} = \frac{1}{p+(1-p)e^{-1/\tau}},
        \quad
        \rho^{\mathrm{part}}_- = \frac{1}{p e^{1/\tau}+(1-p)} = \frac{e^{-1/\tau}}{p+(1-p)e^{-1/\tau}}.
    \end{align*}
    Expanding $(1+u)^{-1}=1-u+O(u^2)$ with $u=\frac{1-p}{p}e^{-1/\tau}$ yields
    \begin{align*}
        \rho^{\mathrm{part}}_+ = \frac{1}{p}-\frac{1-p}{p^2}e^{-1/\tau}+O(e^{-2/\tau}),
        \quad
        \rho^{\mathrm{part}}_- = \frac{1}{p}e^{-1/\tau}+O(e^{-2/\tau}).
    \end{align*}
\end{proof}

\subsection{Proof of Proposition~\ref{prop:mixed_subproblem}}
\label{sec:proof_mixed}
\begin{proof}[Proof of Proposition~\ref{prop:mixed_subproblem}]
    Fix a state and omit $x$. Let $u(y)\coloneqq\log\frac{\pi(y)}{\pi_t(y)}$. 
    Then the simplex constraint $\sum_y\pi(y)=1$ is equivalent to the single equality constraint
    \begin{align}
        \label{eq:u_norm}
        \E_{y\sim\pi_t}\big[e^{u(y)}\big]=1.
    \end{align}
    Hence the mixed subproblem \eqref{eq:mixed_subproblem} is equivalent to
    \begin{align*}
        \max_{u\colon\E_{\pi_t}[e^u]=1}\E_{\pi_t}\Big[e^u r - \tau e^u u - \frac{\lambda}{2\tau}(e^u-1)^2\Big].
    \end{align*}
    Introduce a Lagrange multiplier $\nu\in\mathbb{R}$ for the constraint \eqref{eq:u_norm} and define
    \begin{align*}
        \mathcal{L}(u,\nu)
        \coloneqq\E_{\pi_t}\Big[e^u r - \tau e^u u - \frac{\lambda}{2\tau}(e^u-1)^2\Big] +\nu\big(\E_{\pi_t}[e^u]-1\big).
    \end{align*}
    Stationarity w.r.t. $u(y)$ gives, for all $y$,
    \begin{align*}
        0=\pi_t(y)e^{u(y)}\Big(r(y)-\tau(u(y)+1)-\frac{\lambda}{\tau}(e^{u(y)}-1)+\nu\Big).
    \end{align*}
    Dividing $\pi_t(y)e^{u(y)}>0$ and rearranging the terms give
    \begin{align}\label{eq:u_kkt_shift}
        u(y)-\frac{r(y)}{\tau}+\frac{\lambda}{\tau^2}e^{u(y)}=c,
        \quad
        c\coloneqq\frac{\nu+\lambda/\tau-\tau}{\tau},
    \end{align}
    where $c$ is a constant independent of $y$.
    
    Finally, adding a constant baseline to $r$ does not change the optimizer over the simplex, thus by choosing $b\coloneqq\tau c$ and writing $\Delta_y\coloneqq r(y)-\E_{\pi_t}[r]$ (which differs from $r$ by a constant),
    we may rewrite \eqref{eq:u_kkt_shift} equivalently as
    \begin{align}\label{eq:u_kkt_final}
        u(y)-\frac{\Delta_y}{\tau}+\frac{\lambda}{\tau^2}e^{u(y)}=0, \quad \E_{\pi_t}[e^{u(y)}]=1.
    \end{align}
    These are exactly the same KKT conditions obtained for the \pmdmean\ population objective in \Cref{sec:proof_thm1}. Therefore the \pmdmean\ solution $\pi_{t+1}$ also solves the mixed subproblem \eqref{eq:mixed_subproblem} with the same $\lambda$.
\end{proof}

\section{Missing Proofs in Section~\ref{sec:convergence}}
We first state a simple self-bounding lemma from Bernstein's inequality.

\begin{lemma}[Bernstein with self-bounding variance]\label{lem:bernstein_excess_self_bounding}
    Let $Z_1,\dots,Z_n$ be i.i.d. random variables such that $\E[Z_i]=\mu\geq 0$, $\abs{Z_i}\leq R$, and $\E[Z_i^2]\leq v\mu$ for some $v>0$. 
    Then for any $\delta\in(0,1)$, with probability at least $1-\delta$,
    \begin{align}\label{eq:bernstein_excess_self_bounding}
        \mu\leq 2\cdot \frac{1}{n}\sum_{i=1}^n Z_i +\frac{(2v+\frac{4}{3}R)\log(1/\delta)}{n}.
    \end{align}
\end{lemma}
\begin{proof}
    By Bernstein's inequality for bounded variables, with probability at least $1-\delta$,
    \begin{align*}
        \mu\leq\frac{1}{n}\sum_{i=1}^n Z_i+\sqrt{\frac{2\Var(Z_i)\log(1/\delta)}{n}}+\frac{2R\log(1/\delta)}{3n}.
    \end{align*}
    Using $\Var(Z_i)\leq \E[Z_i^2]\leq v\mu$ yields
    \begin{align*}
        \mu\leq
        \frac{1}{n}\sum_{i=1}^n Z_i +\sqrt{\frac{2v\mu\log(1/\delta)}{n}} +\frac{2R\log(1/\delta)}{3n}.
    \end{align*}
    Apply $\sqrt{ab}\leq \frac{1}{2}a+\frac{1}{2}b$ with
    $a=\mu$ and $b=\frac{2v\log(1/\delta)}{n}$:
    \begin{align*}
        \sqrt{\frac{2v\mu\log(1/\delta)}{n}}\leq\frac{\mu}{2}+\frac{v\log(1/\delta)}{n}.
    \end{align*}
    Substitute and rearrange to obtain \eqref{eq:bernstein_excess_self_bounding}.
\end{proof}

\subsection{Proof of \Cref{thm:erm_loo}}
\label{sec:proof_erm_loo}
\begin{proof}[Proof of \Cref{thm:erm_loo}]
    Fix iteration $t$. For each $\pi\in\Pi$, define the residual
    \begin{align*}
        f_\pi(y)\coloneqq s_\pi(y)-s^\star(y).
    \end{align*}
    By \Cref{asm:realizable_plugin}, $s^\star=s_{\pi^\star_{t+1}}$ for some $\pi^\star_{t+1}\in\Pi$. Hence by \Cref{asm:bounded_logratio},
    for all $\pi\in\Pi$ and $y\in\mathcal{Y}$, $s_\pi(y)\in[-B,B_+]$, 
    \begin{align*}
        \abs{f_\pi(y)} 
        \leq B + B_+
        \eqqcolon M.
    \end{align*}
    Define the clean empirical loss
    \begin{align*}
        \widehat{\mathcal{L}}^{\mathrm{clean}}_t(\pi)
        \coloneqq\frac{1}{2n}\sum_{i=1}^n f_\pi(y_i)^2
        =\frac{1}{n}\sum_{i=1}^n X_i(\pi),
        \quad
        X_i(\pi)\coloneqq \frac{1}{2} f_\pi(y_i)^2.
    \end{align*}
    Then $0\leq X_i(\pi)\leq \frac{1}{2} M^2$ and $\E[X_i(\pi)]=\mathcal{L}_t(\pi)$.
    Applying \Cref{lem:bernstein_excess_self_bounding} with $v=R=\frac{1}{2} M^2$ and a union bound over $\pi\in\Pi$, we get that with probability at least $1-\delta$, for all $\pi\in\Pi$, 
    \begin{align}
    \label{eq:uniform_clean_fast_rate}
        \mathcal{L}_t(\pi)
        \leq 2\widehat{\mathcal{L}}^{\mathrm{clean}}_t(\pi) +\frac{5M^2\log(\abs{\Pi}/\delta)}{3n}
        < 2\widehat{\mathcal{L}}^{\mathrm{clean}}_t(\pi)+\frac{5M^2\log(2\abs{\Pi}/\delta)}{3n}.
    \end{align}
    
    Next, relate $\widehat{\mathcal{L}}^{\mathrm{clean}}_t(\widehat{\pi}_{t+1})$ to the target mismatch $\overline{\Delta^2}$.
    For each $i$, write $\Delta_i=\widetilde{s}^\star_{-i}(y_i)-s^\star(y_i)$ so that
    \begin{align*}
        s_{\widehat{\pi}_{t+1}}(y_i)-\widetilde{s}^\star_{-i}(y_i) =f_{\widehat{\pi}_{t+1}}(y_i)-\Delta_i.
    \end{align*}
    Using \Cref{asm:realizable_plugin,asm:opt_err}, we get 
    \begin{align*}
        \widehat{\mathcal{L}}_t(\widehat{\pi}_{t+1})
        \leq\inf_{\pi\in\Pi}\widehat{\mathcal{L}}_t(\pi)+\epsilon_{\mathrm{opt}}
        \leq\widehat{\mathcal{L}}_t(\pi^\star_{t+1})+\epsilon_{\mathrm{opt}}
        =\frac{1}{2n}\sum_{i=1}^n \Delta_i^2+\epsilon_{\mathrm{opt}}
        =\frac{1}{2}\overline{\Delta^2}+\epsilon_{\mathrm{opt}}.
    \end{align*}
    On the other hand, the pointwise inequality $(u-v)^2\ge \frac{1}{2} u^2-v^2$ implies
    \begin{align*}
        \widehat{\mathcal{L}}_t(\widehat{\pi}_{t+1})
        =\frac{1}{2n}\sum_{i=1}^n \big(f_{\widehat{\pi}_{t+1}}(y_i)-\Delta_i\big)^2
        \geq\frac{1}{4n}\sum_{i=1}^n f_{\widehat{\pi}_{t+1}}(y_i)^2-\frac{1}{2n}\sum_{i=1}^n \Delta_i^2
        =\frac{1}{2}\widehat{\mathcal{L}}^{\mathrm{clean}}_t(\widehat{\pi}_{t+1}) -\frac{1}{2}\overline{\Delta^2}.
    \end{align*}
    Combining the last two displays yields
    \begin{align*}
        \widehat{\mathcal{L}}^{\mathrm{clean}}_t(\widehat{\pi}_{t+1})
        \leq 2\overline{\Delta^2}+2\epsilon_{\mathrm{opt}}.
    \end{align*}
    Finally, apply \eqref{eq:uniform_clean_fast_rate} at $\pi=\widehat{\pi}_{t+1}$ and substitute the above bound, we get
    \begin{align*}
        \mathcal{L}_t(\widehat{\pi}_{t+1})
        \leq 2\widehat{\mathcal{L}}^{\mathrm{clean}}_t(\pi)+\frac{5M^2\log(2\abs{\Pi}/\delta)}{3n}
        \leq 4\overline{\Delta^2}+4\epsilon_{\mathrm{opt}}+\frac{5M^2\log(2\abs{\Pi}/\delta)}{3n},
    \end{align*}
    which proves \eqref{eq:erm_master_bound} as $M\leq 2B$.
\end{proof}

\subsection{Proof of \Cref{cor:convergence_one_step}}
\label{sec:proof_cor_one_step}
\begin{proof}[Proof of \Cref{cor:convergence_one_step}]
    Since $r\in[0,1]$, for any two policies $p,q$ we have $\abs{J(p)-J(q)}\leq \TV(p,q)$.
    Moreover,
    \begin{align*}
        \TV(\pi_{t+1},\pi^\star_{t+1})
        &=\frac{1}{2}\E_{y\sim\pi_t}\left[\abs{e^{s_{\pi_{t+1}}(y)}-e^{s^\star(y)}}\right]\\
        &\leq \frac{1}{2}\E_{y\sim\pi_t}\left[\max\bigl\{e^{s_{\pi_{t+1}}(y)},e^{s^\star(y)}\bigr\}\cdot\abs{s_{\pi_{t+1}}(y)-s^\star(y)}\right]\\
        &\leq \frac{1}{2}\Big(\sqrt{\E_{\pi_t}\big[e^{2s_{\pi_{t+1}}(y)}\big]}+\sqrt{\E_{\pi_t}\big[e^{2s^\star(y)}\big]}\Big)\cdot
        \sqrt{\E_{\pi_t}\left[(s_{\pi_{t+1}}(y)-s^\star(y))^2\right]}\\
        &= \frac{1}{2}\Big(\sqrt{1+\chi^2(\pi_{t+1}~\|~\pi_t)}+\sqrt{1+\chi^2(\pi^\star_{t+1}~\|~\pi_t)}\Big)\cdot \sqrt{2\mathcal{L}_t(\pi_{t+1})},
    \end{align*}
    where we used $\abs{e^a-e^b}\leq \max\{e^a,e^b\}\abs{a-b}$ and Cauchy--Schwarz.
    Since $s_\pi\leq B_+$ and $\E_{\pi_t}[e^{s_\pi}]=1$, we have $\E_{\pi_t}[e^{2s_\pi}]\leq e^{B_+}$, hence the prefactor is at most $e^{B_+/2}$ up to constants.
    Combining the assumption on the ideal convergence rate \eqref{eq:ideal_improvement} and the bound on $\mathcal{L}_t(\pi_{t+1})$ in \Cref{thm:erm_loo}, we obtain the result \eqref{eq:one_step_improvement}.
\end{proof}

\subsection{Proof of Proposition~\ref{prop:eta_mean_asymp}}\label{sec:proof_eta_mean}
\begin{proof}[Proof of Proposition~\ref{prop:eta_mean_asymp}]
    By \eqref{eq:ratio_neg_mean} in \Cref{sec:closed_form}, for $r(y)=0$,
    \begin{align*}
        \frac{\pi^\mathrm{mean}_{t+1}(y)}{\pi_t(y)}=\exp\left(-\frac{p_t}{\tau}\right)\big(1+o(1)\big).
    \end{align*}
    Therefore, the total probability mass on the negative set contracts as
    \begin{align*} 
        1-p_{t+1}^\star
        &=\E_{y\sim\pi_t}\left[\frac{\pi_{t+1}^\star(y)}{\pi_t(y)}\cdot \mathbf{1}\{r(y)=0\}\right]\\
        &=(1-p_t)\exp\left(-\frac{p_t}{\tau}\right)\big(1+o(1)\big),
    \end{align*}
    which yields \eqref{eq:eta_mean_asymp}.
\end{proof}

\subsection{Proof of Proposition~\ref{prop:eta_part_exact}}\label{sec:proof_eta_part}
\begin{proof}[Proof of Proposition~\ref{prop:eta_part_exact}]
    Under \eqref{eq:boltzmann_update}, the total probability mass on $r=1$ is
    \begin{align*}
        p_{t+1}^\star
        =\frac{p_t e^{1/\tau}}{p_t e^{1/\tau}+(1-p_t)},
    \end{align*}
    which implies $1-p_{t+1}^\star = \frac{1-p_t}{p_t e^{1/\tau}+(1-p_t)}$.
    This is exactly \eqref{eq:ideal_improvement} with \eqref{eq:eta_part_exact}.
\end{proof}

\subsection{Proof of \Cref{prop:tgt_err_clean}}
\label{sec:proof_tgt_err_clean}

\begin{lemma}[LOO mean concentration]
\label{lem:loo_mean_bernoulli}
    Let $U_1,\dots,U_n$ be i.i.d. Bernoulli random variables with mean $p$ and define $p_{-i}=\frac{1}{n-1}\sum_{j\neq i}U_j$.
    Then for any $\delta\in(0,1)$, with probability at least $1-\delta$,
    \begin{align*}
        \max_{i\in[n]} \abs{p_{-i}-p}
        \leq\sqrt{\frac{2p(1-p)\log(4n/\delta)}{n-1}}+\frac{2\log(4n/\delta)}{3(n-1)}.
    \end{align*}
\end{lemma}
\begin{proof}
    For any $i$, by Bernstein's inequality for $\sum_{j\neq i} (U_j-p)$,
    with probability at least $1-\delta/(2n)$,
    \begin{align*}
        \abs{p_{-i}-p}
        \leq\sqrt{\frac{2p(1-p)\log(\frac{2n}{\delta})}{n-1}}+\frac{2\log(\frac{2n}{\delta})}{3(n-1)}.
    \end{align*}
    Applying a union bound over $i\in[n]$ and both signs gives the stated bound.
\end{proof}

\begin{proof}[Proof of \Cref{prop:tgt_err_clean}]
    Applying \Cref{lem:loo_mean_bernoulli} to $U_i=r(y_i)$, we obtain $\max_i\abs{p_{-i}-p_t}\leq \varepsilon_n(p_t,\delta)$ with probability at least $1-\delta$.
    
    \noindent\textbf{\pmdpart.} 
    For \pmdpart, write $a=e^{1/\tau}-1$ and $Z_t=1+a p_t$.
    Since $r(y_i)\in\{0,1\}$, we have
    \begin{align*}
        \Delta_i
        =\log(1+a p_t)-\log(1+a p_{-i})
        =\frac{a}{1+a\xi_i}(p_t-p_{-i})
    \end{align*}
    for some $\xi_i$ between $p_t$ and $p_{-i}$ by the mean value theorem. 
    Therefore, 
    \begin{align*}
        \abs{\Delta_i}
        \leq\frac{a}{1+a(p_t-\abs{p_{-i}-p_t})_+}\abs{p_{-i}-p_t}.
    \end{align*}
    Meanwhile, there is always $\tau\log(1+ap)\in[0, 1]$, and thus $\abs{\Delta_i}\leq\frac{1}{\tau}$. 
    In the event that $\max_{i\in[n]}\abs{p_{-i}-p_t}\leq \varepsilon_n(p_t,\delta)$, combining the inequalities, squaring and averaging over $i$ yields \eqref{eq:tgt_err_part_clean}.

    \noindent\textbf{\pmdmean.}
    Recall that $\widetilde{s}^\star_{-i}(y_i)$ is defined in \eqref{eq:loo-s-mean} and $s^\star(y)=\log\frac{\pi^\mathrm{mean}_{t+1}(y)}{\pi_t(y)}$.
    By \eqref{eq:pmdmean_lambertW_form}, we have
    \begin{align*}
        s^\star(y)
        =\frac{r(y)-p_t}{\tau}-\cW\Bigl(\frac{\lambda}{\tau^2}\exp\Bigl(\frac{r(y)-p_t}{\tau}\Bigr)\Bigr).
    \end{align*}
    Therefore, for each $i\in[n]$,
    \begin{align*}
        \Delta_i
        =\widetilde{s}^\star_{-i}(y_i)-s^\star(y_i)
        =\frac{p_t-p_{-i}}{\tau} + \cW\Bigl(\frac{\lambda}{\tau^2}\exp\Bigl(\frac{r(y_i)-p_t}{\tau}\Bigr)\Bigr).
    \end{align*}
    Using $(a+b)^2\leq 2a^2+2b^2$ and averaging over $i$ gives
    \begin{align*}
        \overline{\Delta^2}
        \leq \frac{2\max_i\abs{p_{-i}-p_t}^2}{\tau^2} + \frac{2}{n}\sum_{i=1}^n \cW\Bigl(\frac{\lambda}{\tau^2}\exp\Bigl(\frac{r(y_i)-p_t}{\tau}\Bigr)\Bigr)^2.
    \end{align*}
    On the event $\max_i\abs{p_{-i}-p_t}\leq \varepsilon_n(p_t,\delta)$, the first term is bounded by $\frac{2\varepsilon_n(p_t,\delta)^2}{\tau^2}$.
    For the second term, since $r(y_i)\in\{0,1\}$ it takes only two values:
    \begin{align*}
        w_+\coloneqq\cW\Bigl(\frac{\lambda}{\tau^2}\exp\Bigl(\frac{1-p_t}{\tau}\Bigr)\Bigr),
        \quad
        w_-\coloneqq\cW\Bigl(\frac{\lambda}{\tau^2}\exp\Bigl(-\frac{p_t}{\tau}\Bigr)\Bigr).
    \end{align*}
    Writing $\widehat{p}_t\coloneqq \frac{1}{n}\sum_{i=1}^n r(y_i)$, we have
    \begin{align*}
        \frac{1}{n}\sum_{i=1}^n\cW\Bigl(\frac{\lambda}{\tau^2}\exp\Bigl(\frac{r(y_i)-p_t}{\tau}\Bigr)\Bigr)^2
        =\widehat{p}_t w_+^2 + (1-\widehat{p}_t) w_-^2.
    \end{align*}
    By the asymptotics used in the proof of \Cref{thm:pmdmean_solution} in \Cref{sec:proof_thm1}, as $\tau\to 0$ with fixed $p_t\in(0,1)$ we have $\tau w_+=(1-p_t)+o(1)$ and $\tau w_-=o(1)$, and hence for sufficiently small $\tau$,
    \begin{align*}
        w_+^2 \lesssim \frac{(1-p_t)^2}{\tau^2},
        \quad
        w_-^2 = o\Bigl(\frac{1}{\tau^2}\Bigr).
    \end{align*}
    Since $\max_i\abs{p_{-i}-p_t}\leq\varepsilon_n(p_t,\delta)$, we have 
    \begin{align*}
        \abs{\widehat{p}_t-p_t}
        =\abs{\frac{1}{n}\sum_{i=1}^n(p_{-i} - p_t)}
        \leq\frac{1}{n}\sum_{i=1}^n\abs{p_{-i} - p_t}
        \leq\varepsilon_n(p_t,\delta). 
    \end{align*}
    Combining the above bounds yields
    \begin{align*}
        \frac{1}{n}\sum_{i=1}^n\cW\Bigl(\frac{\lambda}{\tau^2}\exp\Bigl(\frac{r(y_i)-p_t}{\tau}\Bigr)\Bigr)^2
        \lesssim \frac{p_t(1-p_t)^2}{\tau^2}\left(1+\varepsilon_n(p_t,\delta)\right).
    \end{align*}
    Substituting back proves \eqref{eq:tgt_err_mean_clean}.
\end{proof}

\section{Refined Analysis for \pmdmean}\label{sec:refined}
To refine the analysis of \pmdmean, we first connect the population squared loss $\mathcal{L}_t$ in \eqref{eq:population-loss} with ideal target $s^\star(y)=\log\frac{\pi_{t+1}^\mathrm{mean}(y)}{\pi_t(y)}$ with the population objective of \pmdmean\ in \eqref{eq:L_mean}. 
\begin{lemma}[Connection of losses for \pmdmean]\label{lem:pmdmean_quad_growth}
    Fix a global step $t$ and write $\Delta_y\coloneqq r(y)-\E_{y^\prime\sim\pi_t}[r(y^\prime)]$.
    Define the \pmdmean\ population objective, i.e., bandit specialization of \eqref{eq:L_mean}:
    \begin{align}\label{eq:L_mean_t}
        \mathcal{L}^{\mathrm{mean}}_t(\pi)
        \coloneqq\frac{1}{2}\E_{y\sim\pi_t}\Bigl[\Bigl(s_\pi(y)-\frac{\Delta_y}{\tau}\Bigr)^2\Bigr].
    \end{align}
    Let $\pi^\star_{t+1}$ be the ideal \pmdmean\ update and $s^\star=s_{\pi^\star_{t+1}}$.
    Then for any $\pi\in\Pi$,
    \begin{align}\label{eq:pmdmean_quad_growth}
        \mathcal{L}^{\mathrm{mean}}_t(\pi)-\mathcal{L}^{\mathrm{mean}}_t(\pi^\star_{t+1}) =\mathcal{L}_t(\pi) + \frac{\lambda}{\tau^2}\KL(\pi^\star_{t+1}~\|~\pi)\geq\mathcal{L}_t(\pi),
    \end{align}
    where $\lambda\geq 0$ is the KKT dual multiplier in \eqref{eq:u_kkt_final}.
\end{lemma}

Using \Cref{lem:pmdmean_quad_growth}, we can refine the ERM analysis for \pmdmean\ and eliminate the error floor in target estimation error. 
\begin{lemma}[Refined ERM for \pmdmean]\label{lem:erm_pmdmean_constrained_refined}
    Suppose \Cref{asm:realizable_plugin,asm:opt_err,asm:bounded_logratio,asm:finite_class} hold, $r(y)\in\{0,1\}$ and define $p_t\coloneqq \E_{y\sim\pi_t}[r(y)]$.
    Let $\varepsilon_n(p_t,\delta)$ be as in \Cref{prop:tgt_err_clean}.
    Then for any $\delta\in(0,1)$, with probability at least $1-\delta$,
    \begin{align}\label{eq:erm_pmdmean_constrained_refined}
        \mathcal{L}_t(\widehat{\pi}_{t+1})
        \lesssim
        \frac{(B+\frac{1}{\tau})^2\log(\abs{\Pi}/\delta)}{n} +\epsilon_{\mathrm{opt}} +\frac{\varepsilon_n(p_t,\delta)}{\tau}\Bigl(B+\frac{p_t}{\tau}\Bigr) +\frac{\varepsilon_n(p_t,\delta)^2}{\tau^2}.
    \end{align}
\end{lemma}
By substituting $B=\frac{p_t}{\tau}$ for \pmdmean, we obtain \Cref{lem:pmdmean_refined} in \Cref{sec:theory}.

\subsection{Proof of \Cref{lem:pmdmean_quad_growth}}
\label{sec:proof_pmdmean_quad_growth}

\begin{proof}[Proof of \Cref{lem:pmdmean_quad_growth}]
    Denote $g(y)\coloneqq \Delta_y/\tau$ for brevity.
    For any $\pi\in\Pi$, we have
    \begin{align*}
        \mathcal{L}^{\mathrm{mean}}_t(\pi)-\mathcal{L}^{\mathrm{mean}}_t(\pi^\star_{t+1})
        &=\frac{1}{2}\E_{\pi_t}\big[(s_\pi-g)^2-(s^\star-g)^2\big]\\
        &=\frac{1}{2}\E_{\pi_t}\big[(s_\pi-s^\star)^2\big]
        +\E_{\pi_t}\big[(s_\pi-s^\star)(s^\star-g)\big]\\
        &=\mathcal{L}_t(\pi)+\E_{\pi_t}\big[(s_\pi-s^\star)(s^\star-g)\big].
    \end{align*}
    By the KKT conditions \eqref{eq:u_kkt_final} with $u=s^\star$,
    \begin{align*}
        s^\star(y)-g(y)=-\frac{\lambda}{\tau^2}e^{s^\star(y)}.
    \end{align*}
    Combining the identities and $e^{s^\star(y)}=\pi^\star_{t+1}(y)/\pi_t(y)$, we get
    \begin{align*}
        \E_{\pi_t}\big[(s_\pi-s^\star)(s^\star-g)\big]
        &=-\frac{\lambda}{\tau^2}\E_{\pi_t}\big[(s_\pi-s^\star)e^{s^\star}\big]\\
        &=-\frac{\lambda}{\tau^2}\E_{y\sim\pi^\star_{t+1}}\Bigl[\log\frac{\pi(y)}{\pi^\star_{t+1}(y)}\Bigr]\\
        &=\frac{\lambda}{\tau^2}\KL(\pi^\star_{t+1}~\|~\pi)
        \geq 0,
    \end{align*}
    which proves \eqref{eq:pmdmean_quad_growth}.
\end{proof}

\subsection{Proof of \Cref{lem:erm_pmdmean_constrained_refined}}
\label{sec:proof_erm_pmdmean_constrained_refined}
\begin{proof}[Proof of \Cref{lem:erm_pmdmean_constrained_refined}]
    Recall the \pmdmean\ leave-one-out target \eqref{eq:loo-s-mean}:
    \begin{align*}
        \widetilde{s}^\star_{-i}(y_i)=\frac{1}{\tau}\big(r(y_i)-p_{-i}\big),
        \quad
        p_{-i}\coloneqq \frac{1}{n-1}\sum_{j\neq i} r(y_j).
    \end{align*}
    Also recall $p_t=\E_{y\sim\pi_t}[r(y)]$ and $\Delta_{y_i}=r(y_i)-p_t$.

    % \paragraph{Step 1: Decompose the empirical targets.}
    We first decompose the empirical targets as 
    \begin{align}\label{eq:target_decomp_pmdmean}
        \widetilde{s}^\star_{-i}(y_i)
        =\frac{\Delta_{y_i}}{\tau} +\Delta^{\mathrm{loo}}_i,
        \quad
        \Delta^{\mathrm{loo}}_i \coloneqq \frac{p_t-p_{-i}}{\tau}.
    \end{align}
    Define the empirical loss with the population baseline target:
    \begin{align}\label{eq:ghost_emp_loss_mean}
        \widehat{\mathcal{L}}^{\mathrm{mean}}_t(\pi) \coloneqq\frac{1}{2n}\sum_{i=1}^n\Bigl(s_\pi(y_i)-\frac{\Delta_{y_i}}{\tau}\Bigr)^2.
    \end{align}
    By \eqref{eq:target_decomp_pmdmean}, for each $\pi\in\Pi$,
    \begin{align*}
        \widehat{\mathcal{L}}_t(\pi)-\widehat{\mathcal{L}}^{\mathrm{mean}}_t(\pi)
        &=\frac{1}{2n}\sum_{i=1}^n\Bigl[\bigl(a_i-\Delta^{\mathrm{loo}}_i\bigr)^2-a_i^2\Bigr]\\
        &=\frac{1}{2n}\sum_{i=1}^n\Bigl[(\Delta^{\mathrm{loo}}_i)^2-2a_i\cdot\Delta^{\mathrm{loo}}_i\Bigr],
    \end{align*}
    where $a_i\coloneqq s_\pi(y_i)-\frac{\Delta_{y_i}}{\tau}$. 
    Thus we have 
    \begin{align}
    \label{eq:loo_perturb_identity}
        \abs{\widehat{\mathcal{L}}_t(\pi)-\widehat{\mathcal{L}}^{\mathrm{mean}}_t(\pi)}
        \leq
        \frac{1}{n}\sum_{i=1}^n \abs{a_i}\cdot \abs{\Delta^{\mathrm{loo}}_i}
        +\frac{1}{2n}\sum_{i=1}^n(\Delta^{\mathrm{loo}}_i)^2.
    \end{align}
    
    Let $\mathcal{E}$ be the event from \Cref{prop:tgt_err_clean} that
    $\max_i \abs{p_{-i}-p_t}\leq \varepsilon_n(p_t,\delta/2)$.
    On $\mathcal{E}$ we have $\abs{\Delta^{\mathrm{loo}}_i}\leq \varepsilon_n(p_t,\delta/2)/\tau$, hence
    \begin{align}\label{eq:loo_perturb_on_E}
        \sup_{\pi\in\Pi}
        \abs{\widehat{\mathcal{L}}_t(\pi)-\widehat{\mathcal{L}}^{\mathrm{mean}}_t(\pi)}
        \leq
        \frac{\varepsilon_n}{\tau}\cdot
        \sup_{\pi\in\Pi}\frac{1}{n}\sum_{i=1}^n \abs{a_i}
        +\frac{\varepsilon_n^2}{2\tau^2},
    \end{align}
    where we write $\varepsilon_n\coloneqq \varepsilon_n(p_t,\delta/2)$.
    
    We now bound the $\frac{1}{n}\sum_{i=1}^n\abs{a_i}$ term. 
    By \Cref{asm:bounded_logratio}, $\abs{s_\pi(y)}\leq B$ for all $\pi\in\Pi, y\in\mathcal{Y}$.
    Also, since $r(y_i)\in\{0,1\}$,
    \begin{align*}
        \abs{\Delta_{y_i}}
        =\abs{r(y_i)-p_t}
        =
        \begin{cases}
        p_t,& r(y_i)=0,\\
        1-p_t,& r(y_i)=1.
        \end{cases}
    \end{align*}
    Thus for any $\pi$ and any $i$,
    \begin{align*}
        \abs{a_i}
        =\abs{s_\pi(y_i)-\frac{\Delta_{y_i}}{\tau}}
        \leq \abs{s_\pi(y_i)}+\frac{\abs{\Delta_{y_i}}}{\tau}
        \leq B+\frac{\abs{r(y_i)-p_t}}{\tau}.
    \end{align*}
    Averaging over $i$ yields
    \begin{align}\label{eq:avg_ai_bound}
        \frac{1}{n}\sum_{i=1}^n \abs{a_i}
        \leq B+\frac{1}{\tau}\cdot \frac{1}{n}\sum_{i=1}^n \abs{r(y_i)-p_t}.
    \end{align}
    Let $\bar{r}\coloneqq \frac{1}{n}\sum_{i=1}^n r(y_i)$. For binary rewards,
    \begin{align*}
        \frac{1}{n}\sum_{i=1}^n \abs{r(y_i)-p_t}
        = p_t+(1-2p_t)\bar{r}
        \leq 2p_t+\abs{\bar{r}-p_t}.
    \end{align*}
    Moreover, for any fixed $i$,
    \begin{align*}
        \bar{r}=\frac{(n-1)p_{-i}+r(y_i)}{n},
    \end{align*}
    hence
    \begin{align*}
        \abs{\bar{r}-p_t}
        \leq
        \frac{n-1}{n}\abs{p_{-i}-p_t}+\frac{1}{n}\abs{r(y_i)-p_t}
        \leq
        \max_{j} \abs{p_{-j}-p_t}+\frac{1}{n}.
    \end{align*}
    On $\mathcal{E}$ this gives $\abs{\bar{r}-p_t}\leq \varepsilon_n+\frac{1}{n}$, and therefore
    \begin{align}\label{eq:avg_abs_centered_r}
        \frac{1}{n}\sum_{i=1}^n \abs{r(y_i)-p_t}
        \leq
        2p_t+\varepsilon_n+\frac{1}{n}.
    \end{align}
    Combining \eqref{eq:avg_ai_bound} and \eqref{eq:avg_abs_centered_r} and substituting into \eqref{eq:loo_perturb_on_E} yields
    \begin{align}
        \label{eq:loo_perturb_final}
        \sup_{\pi\in\Pi}
        \abs{\widehat{\mathcal{L}}_t(\pi)-\widehat{\mathcal{L}}^{\mathrm{mean}}_t(\pi)}
        &\lesssim\frac{\varepsilon_n}{\tau}\Bigl(B+\frac{p_t}{\tau} + \frac{\varepsilon_n}{\tau} + \frac{1}{n\tau}\Bigr)\nonumber\\
        &\lesssim\frac{\varepsilon_n}{\tau}\Bigl(B+\frac{p_t}{\tau}\Bigr)
        +\frac{\varepsilon_n^2}{\tau^2}.
    \end{align}
    
    We now bound the population excess risk of \pmdmean. 
    Recall that 
    \begin{align*}
        &\mathcal{L}^{\mathrm{mean}}_t(\pi)
        \coloneqq\frac{1}{2}\E_{y\sim\pi_t}\Bigl[\Bigl(s_\pi(y)-\frac{\Delta_y}{\tau}\Bigr)^2\Bigr],\\
        &\widehat{\mathcal{L}}^{\mathrm{mean}}_t(\pi) \coloneqq\frac{1}{2n}\sum_{i=1}^n\Bigl(s_\pi(y_i)-\frac{\Delta_{y_i}}{\tau}\Bigr)^2.
    \end{align*}
    For $\pi\in\Pi$, define the pointwise loss
    \begin{align*}
        &\ell_\pi(y)\coloneqq \frac{1}{2}\Bigl(s_\pi(y)-\frac{\Delta_y}{\tau}\Bigr)^2,\quad 
        \ell^\star(y)\coloneqq \frac{1}{2}\Bigl(s^\star(y)-\frac{\Delta_y}{\tau}\Bigr)^2.
    \end{align*}
    Let
    \begin{align*}
        Z_i(\pi)\coloneqq \ell_\pi(y_i)-\ell^\star(y_i),
        \quad
        \mu(\pi)\coloneqq \E_{\pi_t}[Z_i(\pi)]
        =
        \mathcal{L}^{\mathrm{mean}}_t(\pi)-\mathcal{L}^{\mathrm{mean}}_t(\pi^\star_{t+1}),
        \quad
        \widehat{\mu}(\pi)\coloneqq \frac{1}{n}\sum_{i=1}^n Z_i(\pi).
    \end{align*}
    By \Cref{asm:bounded_logratio} and $\abs{\Delta_y}\leq 1$, we have
    \begin{align*}
        0
        \leq \ell_\pi(y)
        \leq \frac{1}{2} M_\tau^2,
    \end{align*}
    where $M_\tau\coloneqq B+\frac{1}{\tau}$, and hence $\abs{Z_i(\pi)}\leq \frac{1}{2} M_\tau^2$. 
    Moreover, 
    \begin{align*}
        Z_i^2(\pi)
        &=\frac{1}{4}(s^\star(y)-s_\pi(y))^2\left(s^\star(y)+s_\pi(y) - \frac{2\Delta_y}{\tau}\right)^2\\
        &\leq (s^\star(y)-s_\pi(y))^2 M_\tau^2,
    \end{align*}
    thus we have
    \begin{align*}
        \E[Z_i^2(\pi)]
        &\leq M_\tau^2\E[(s_\pi-s^\star)^2]\\
        &=2M_\tau^2\mathcal{L}_t(\pi)\\
        &\leq 2M_\tau^2\mu(\pi),
    \end{align*}
    where the last inequality comes from \Cref{lem:pmdmean_quad_growth}. 
    Apply \Cref{lem:bernstein_excess_self_bounding} with $v=2M_\tau^2$ and $R=\frac{1}{2} M_\tau^2$, then with probability at least $1-\delta^\prime$, for a fixed $\pi$, 
    \begin{align*}
        \mu(\pi)
        \leq 2\widehat{\mu}(\pi)+ O\left(\frac{M_\tau^2\log(1/\delta^\prime)}{n}\right).
    \end{align*}
    A union bound over $\Pi$ with $\delta^\prime=\frac{\delta}{2\abs{\Pi}}$ yields that with probability at least $1-\frac{\delta}{2}$, for all $\pi\in\Pi$,
    \begin{align}\label{eq:uniform_excess_mean}
        \mathcal{L}_t(\pi)
        \leq\mathcal{L}^{\mathrm{mean}}_t(\pi)-\mathcal{L}^{\mathrm{mean}}_t(\pi^\star_{t+1})
        \lesssim
        \widehat{\mathcal{L}}^{\mathrm{mean}}_t(\pi)-\widehat{\mathcal{L}}^{\mathrm{mean}}_t(\pi^\star_{t+1})
        +\frac{M_\tau^2\log(\abs{\Pi}/\delta)}{n}.
    \end{align}

    We now bound the empirical excess risk. 
    \begin{align*}
        &\widehat{\mathcal{L}}^{\mathrm{mean}}_t(\widehat{\pi}_{t+1}) -\widehat{\mathcal{L}}^{\mathrm{mean}}_t(\pi^\star_{t+1})\\
        =&\widehat{\mathcal{L}}^{\mathrm{mean}}_t(\widehat{\pi}_{t+1}) 
        - \widehat{\mathcal{L}}_t(\widehat{\pi}_{t+1}) + \underbrace{\widehat{\mathcal{L}}_t(\widehat{\pi}_{t+1})  - \widehat{\mathcal{L}}_t(\pi^\star_{t+1})}_{\leq \epsilon_\mathrm{opt}} + \widehat{\mathcal{L}}_t(\pi^\star_{t+1}) 
        -\widehat{\mathcal{L}}^{\mathrm{mean}}_t(\pi^\star_{t+1})\\
        \leq & 
        \epsilon_{\mathrm{opt}}
        +2\sup_{\pi\in\Pi}\abs{\widehat{\mathcal{L}}_t(\pi)-\widehat{\mathcal{L}}^{\mathrm{mean}}_t(\pi)}\\
        \lesssim& \epsilon_{\mathrm{opt}}
        +\frac{\varepsilon_n}{\tau}\Bigl(B+\frac{p_t}{\tau}\Bigr)
        +\frac{\varepsilon_n^2}{\tau^2},
    \end{align*}
    where the first inequality uses \Cref{asm:opt_err,asm:realizable_plugin} and the second inequality uses \eqref{eq:loo_perturb_final}. 
    Combining the above inequality with \eqref{eq:uniform_excess_mean}, we get 
    \begin{align*}
        \mathcal{L}_t(\pi)
        \lesssim \frac{M_\tau^2\log(\abs{\Pi}/\delta)}{n} + \epsilon_{\mathrm{opt}}
        +\frac{\varepsilon_n}{\tau}\Bigl(B+\frac{p_t}{\tau}\Bigr)
        +\frac{\varepsilon_n^2}{\tau^2},
    \end{align*}
    which is \eqref{eq:erm_pmdmean_constrained_refined}. 
\end{proof}

\end{document}